\documentclass{article} 
\usepackage{iclr2027_conference,times}

\usepackage{amsmath,amsfonts,bm}

\def\eqref#1{equation~\ref{#1}}

\def\1{\bm{1}}

\DeclareMathAlphabet{\mathsfit}{\encodingdefault}{\sfdefault}{m}{sl}
\SetMathAlphabet{\mathsfit}{bold}{\encodingdefault}{\sfdefault}{bx}{n}

\usepackage{hyperref}
\usepackage{url}

\usepackage{amsmath}
\usepackage{amssymb}
\usepackage{mathtools}
\usepackage{amsthm}
\usepackage[capitalize]{cleveref}
\usepackage{wrapfig}

\usepackage{subcaption}
\usepackage{algorithm}
\usepackage{algpseudocode}
\usepackage{makecell}
\usepackage{tabularx}
\usepackage{array}
\usepackage{float}
\theoremstyle{plain}
\newtheorem{theorem}{Theorem}[section]
\newtheorem{proposition}[theorem]{Proposition}

\newtheorem{corollary}[theorem]{Corollary}
\theoremstyle{definition}

\newtheorem{assumption}[theorem]{Assumption}
\theoremstyle{remark}

\usepackage[textsize=tiny]{todonotes}
\usepackage{booktabs}       
\usepackage{amsfonts}       
\usepackage{nicefrac}       
\usepackage{microtype}      
\usepackage[table]{xcolor}         
\crefname{assumption}{Assumption}{Assumptions}
\Crefname{assumption}{Assumption}{Assumptions}
\title{Reprogramming Vision-Language Models via Structured Prompt Reparameterization}

\iclrfinalcopy

\author{
Zizhao Li,
Chengyi Cai,
Mohammed Yaqoob Ansari,
Feng Liu,
Joseph West, \\ ~\textbf{Kourosh Khoshelham}\\
The University of Melbourne, Melbourne, Australia
}

\begin{document}

\maketitle

\begin{abstract}
Visual reprogramming adapts pretrained models to downstream tasks by modifying their input and output interfaces while keeping the backbone fixed. In vision-language models, existing methods mainly rely on intra-class prompt aggregation and do not explicitly model relationships among classes. However, fine-grained categories often exhibit highly overlapping attribute descriptions and strong inter-class correlation in the text embedding space, where discriminative cues lie in subtle low-variance components. We propose Reparameterized Inter-Class Visual Reprogramming (RVP), a structured framework that aggregates multiple text prompts within each class and applies residual correction across classes. We also show that CLIP-based visual reprogramming with input-independent linear output aggregation can be expressed as a linear mapping from frozen image embeddings to downstream logits, and use this view to design a structured reparameterization that models shared semantic components and class-specific differences. RVP uses only a single visual prompt and can be reparameterized at inference into a frozen backbone followed by a linear classifier, incurring nearly zero computational overhead. Across 11 few-shot classification benchmarks and four CLIP backbones, RVP consistently improves over prior visual reprogramming methods with comparable or better inference efficiency.
\end{abstract}

\section{Introduction}

Model reprogramming~\citep{vinod2020reprogramming,chen2024model,hung2023low} adapts a pretrained model to downstream tasks by modifying its input and output interfaces while keeping the pretrained parameters fixed. In vision, this is often instantiated as visual reprogramming (VR)~\citep{cai2024sample,cai2024bayesian,tsao2023autovp,chen2023understanding,elsayedadversarial,tsai2020transfer}, where a trainable input transformation is learned while the backbone remains frozen. For vision-language models~\citep{radford2021learning,pmlr-v139-jia21b}, this paradigm is particularly appealing for few-shot adaptation because it preserves pretrained representations, requires only a small number of trainable parameters, and adds little inference overhead.

In CLIP-based VR, an input image is transformed by a visual prompt and encoded by the frozen image encoder, while downstream classes are represented by text embeddings from the frozen text encoder. Classification is then performed through image-text similarity. Recent methods improve this pipeline by introducing multiple textual descriptions for each class and aggregating their similarity scores~\citep{cai2025attribute,cai2025understanding,wu2026dga}. However, these methods remain limited to intra-class prompt selection and do not explicitly model relationships among classes.

This limitation becomes severe in fine-grained recognition. As shown in \cref{fig:svd}, visually similar categories often share highly overlapping attribute descriptions, and prompt groups from different classes can exhibit high cosine similarity. This suggests that attribute prompts are often highly similar and that independent intra-class prompt selection is insufficient to resolve cross-class ambiguity. Moreover, the text embedding matrix exhibits a rapidly decaying singular value spectrum, indicating strong inter-class correlation and a low effective rank. In other words, many classes share dominant semantic directions, while the truly discriminative cues lie in subtle, low-variance components. These observations motivate explicit modeling of class relationships.

\begin{figure*}
    \centering
    \includegraphics[width=\linewidth]{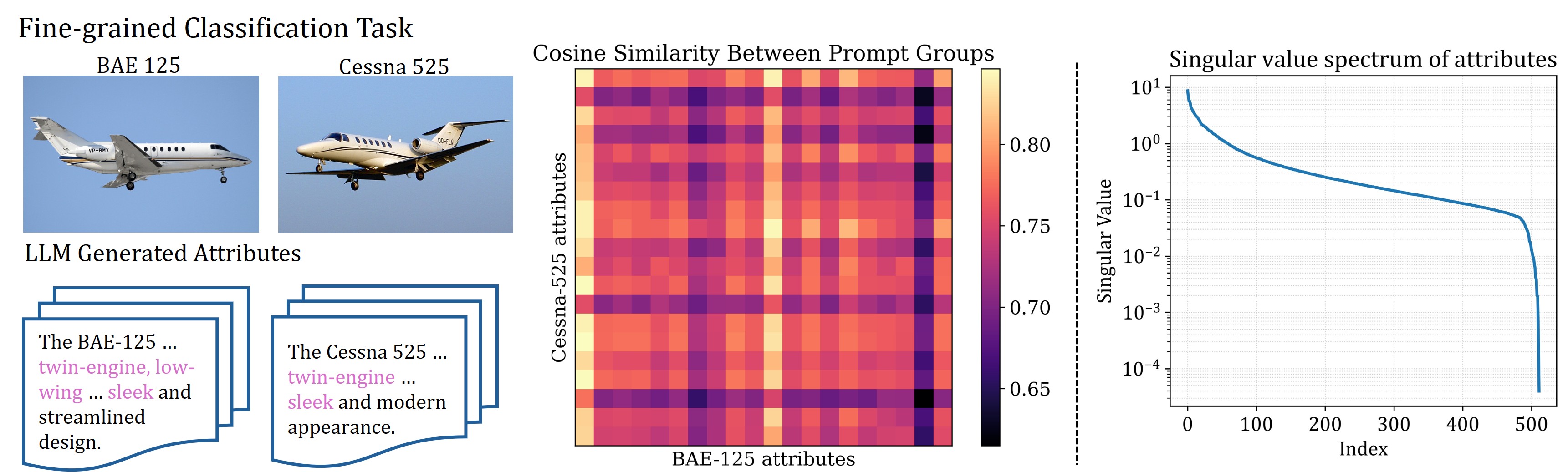}
    \vspace{-0.4cm}
    \caption{Fine-grained classification with attribute prompts. Prompt groups from visually similar classes show high cosine similarity, indicating that attribute prompts are highly similar. The rapidly decaying singular value spectrum of the text embedding matrix further reveals strong inter-class correlation, motivating explicit modeling of class relationships.}
    \label{fig:svd}
    \vspace{-0.4cm}
\end{figure*}

Based on this observation, we propose \emph{Reparameterized Inter-Class Visual Reprogramming} (RVP), a structured framework that jointly models intra-class aggregation and inter-class residual correction. RVP first learns how to combine multiple attribute descriptions within each class, and then refines the resulting class logits through an inter-class relation matrix. This design is effective for two reasons. First, it can suppress shared semantic components and amplify subtle class-specific differences, which are exactly the cues that matter in fine-grained recognition. Second, by restricting the mapping to a structured residual form, RVP preserves the pretrained semantic subspace of CLIP and avoids the overfitting risk of dense unconstrained mappings. Unlike the prior Decoupled Visual Reprogramming~\citep{cai2025understanding}, which relies on multiple visual prompts and repeated backbone forward passes, RVP uses only a single visual prompt and can be reparameterized~\citep{repvgg,luo2023towards} into a frozen backbone followed by a single linear classifier, incurring nearly zero additional overhead at inference. Here, reparameterization means folding multiple parameter groups into one equivalent matrix.

In summary, our contributions are threefold. First, we show that, given a visual prompt, CLIP-based visual reprogramming reduces to a linear mapping from normalized CLIP image embeddings to downstream logits, providing a unified view of prompt aggregation and label mapping strategies. Second, we propose RVP, a reparameterizable inter-class modeling framework that introduces a structured parameterization of this mapping. Specifically, RVP first constructs class logits through text-embedding-based intra-class attribute aggregation, and then refines them with a residual inter-class correction matrix. Unlike unconstrained linear classifiers or generic logit adapters, RVP keeps the classifier anchored in the CLIP text-embedding subspace while reducing confusion among fine-grained classes. Third, through extensive experiments, we show that this structured inter-class parameterization improves few-shot visual reprogramming while preserving efficient inference.

In \cref{sec:pre}, we show that, given a visual prompt, CLIP-based visual reprogramming induces a linear mapping \(\phi:\mathbb{R}^{D}\rightarrow\mathbb{R}^{C}\) from normalized reprogrammed-image embeddings to downstream logits, which enables reparameterizable inter-class modeling. In \cref{sec:method}, we present RVP, including its training-time formulation and exact inference-time reparameterization into a single linear classifier. \cref{sec:exp,sec:results} then present the experimental setup and results, showing that RVP consistently improves performance across multiple CLIP backbones, especially on fine-grained datasets such as Aircraft and Cars, while maintaining comparable or lower inference cost.

\section{Related Work}
\textbf{Model Reprogramming.}
Model reprogramming~\citep{chen2024model} adapts pretrained models to downstream tasks by learning transformations at the input and output interfaces, without modifying internal parameters. This strategy preserves pretrained knowledge and avoids catastrophic forgetting~\citep{kirkpatrick2017overcoming}, while enabling architecture-agnostic transfer with few trainable parameters. It has been applied to vision~\citep{chen2023understanding,tsai2020transfer,cai2024bayesian,jin2025lorvp}, graph~\citep{jing2023deep}, acoustic~\citep{yang2021voice2series,yang2023english,hung2023low,yen2023neural}, and language models~\citep{hambardzumyan2021warp,vinod2020reprogramming,jin2024time}, with recent work studying its robustness~\citep{chen2025refine,zheng2025endow}.

Input visual reprogramming (VR)~\citep{cai2024sample,cai2024bayesian,chen2023understanding} is a common instantiation for image classification, where learnable patterns are injected into the input space. Typical designs include padded regions~\citep{chen2023understanding,tsai2020transfer,tsao2023autovp} or watermark-style perturbations~\citep{bahng2022exploring,oh2023blackvip}, and have been successfully extended to vision--language models~\citep{oh2023blackvip,Zhang_2024_CVPR}.

\textbf{Prompt Learning.}
Prompt learning introduces trainable parameters directly into a pretrained model, often in an architecture-dependent manner. Prompts can take the form of textual tokens~\citep{zhou2022learning,zhou2022conditional}, visual prompts on images~\citep{chen2023understanding,oh2023blackvip,tsao2023autovp}, internal token prompts~\citep{wang2023transhp}, or cross-modal mappings~\citep{khattak2023maple}.

Applying visual prompts to images for adapting VLMs is functionally equivalent to VR. Existing methods typically learn a single shared prompt, such as watermark overlays~\citep{bahng2022exploring}, padded patterns~\citep{tsai2020transfer,chen2023understanding}, BlackVIP~\citep{oh2023blackvip}, DAM~\citep{huang2023diversity}. Recently, AttrVR~\citep{cai2025attribute} incorporates multiple attribute prompts for each class to improve image–text alignment. DVP~\citep{cai2025understanding} learns multiple prompts with distinct roles, improving learning capacity.

\textbf{Feature Adapter.}
Few-shot CLIP adaptation can modify either features or classifiers. CLIP-Adapter~\citep{clip-adapter}, Tip-Adapter~\citep{tip_adapter}, and Proto-CLIP~\citep{proto_clip} incorporate downstream visual features, while LDC~\citep{logitdeconfusion_2025_CVPR} further combines multi-level feature adaptation with sample-dependent logit correction. TaskRes~\citep{yu2023task} learns a residual directly in the classifier space, and LP++~\citep{huang2024lp} constructs text-informed classifiers from visual prototypes and text embeddings. In contrast, visual reprogramming preserves the pretrained model and adapts the task interface through input prompting and output label mapping. RVP strengthens the latter by learning structured transformations over text-derived responses.

\section{Preliminaries and Insights}
\label{sec:pre}

CLIP~\citep{radford2021learning} consists of an image encoder \(f_{\mathrm{img}}\) and a text encoder \(f_{\mathrm{txt}}\), which map inputs into a shared embedding space \(\mathcal{Z} \subseteq \mathbb{R}^{D}\), where \(D\) is the embedding dimension. Let \(\mathcal{X}^{\mathrm{S}}\) denote the source image space of CLIP, \(\mathcal{V}\) the text space, \(x^{\mathrm{S}} \in \mathcal{X}^{\mathrm{S}}\) an input image, and \(V \in \mathcal{V}\) a text description. The corresponding \(\ell_2\)-normalized image and text embeddings are
\begin{equation}
    \hat{\mathbf{v}} = \frac{f_{\mathrm{img}}(x^{\mathrm{S}})}{\|f_{\mathrm{img}}(x^{\mathrm{S}})\|_2},
    \quad
    \hat{\mathbf{t}} = \frac{f_{\mathrm{txt}}(V)}{\|f_{\mathrm{txt}}(V)\|_2}.
\end{equation}
CLIP computes the image--text similarity as $f_{\mathrm{clip}}(x^{\mathrm{S}}, V) = \frac{1}{\tau}\hat{\mathbf{v}}^\top \hat{\mathbf{t}},$
where \(\tau\) is the temperature.

For a downstream task defined on \(\mathcal{X}^{\mathrm{T}} \times \mathcal{Y}^{\mathrm{T}}\), let \(x^{\mathrm{T}} \in \mathcal{X}^{\mathrm{T}}\) denote an input image and \(\mathcal{Y}^{\mathrm{T}}=\{1,\dots,C\}\) the label set with \(C\) classes. Each class \(y^{\mathrm{T}} \in \mathcal{Y}^{\mathrm{T}}\) is represented by a set of textual descriptions \(\mathcal{A}(y^{\mathrm{T}})\subseteq\mathcal{V}\), and let \(\mathcal{A}=\bigcup_{y^{\mathrm{T}}\in\mathcal{Y}^{\mathrm{T}}}\mathcal{A}(y^{\mathrm{T}})\). The class logit is computed by aggregating similarity scores:
\begin{equation}
\label{eq:clip_logits}
    \big[f_{\mathrm{logits}}(x^{\mathrm{T}};\mathcal{A})\big]_{y^{\mathrm{T}}}
    =
    \operatorname{agg}_{a \in \mathcal{A}(y^{\mathrm{T}})} f_{\mathrm{clip}}(x^{\mathrm{T}}, a),
\end{equation}
where \(\operatorname{agg}(\cdot)\) is an aggregation operator~\citep{cai2025understanding}.

Visual reprogramming adapts downstream images to a frozen pretrained model by learning an input transformation instead of modifying model parameters. Specifically, a trainable transformation \(f_{\mathrm{in}}(\cdot \mid \delta)\), parameterized by a visual prompt \(\delta\), maps the downstream image \(x^{\mathrm{T}}\) to a reprogrammed input that is compatible with CLIP. The reprogrammed image is then processed by the frozen CLIP encoders.

For a fixed image \(x^{\mathrm{T}}\), let \(T \in \mathbb{R}^{|\mathcal{A}| \times D}\) be the matrix of stacked normalized text embeddings:
$T = [\hat{\mathbf{t}}_{a_1}, \dots, \hat{\mathbf{t}}_{a_{|\mathcal{A}|}}]^\top.$
The similarity scores over all descriptions can then be written as $$M_a = \frac{1}{\tau}T\hat{\mathbf{v}}.$$

\paragraph{From Similarity Scores to Class Logits.}
For aggregation operators that can be represented by
input-independent linear weights, the description-level similarity
scores can be mapped to class logits through $\mathbf{M}_y = \mathbf{M}_a^\top \boldsymbol{\omega}$.
where \(\omega \in \mathbb{R}^{|\mathcal{A}| \times C}\) is a reweighting matrix induced by \(\operatorname{agg}(\cdot)\). Substituting \(M_a\), we obtain 
$M_y = \frac{1}{\tau}\hat{\mathbf{v}}^\top T^\top \omega = \phi(\hat{\mathbf{v}}), $
where \(\phi: \mathbb{R}^{D} \to \mathbb{R}^{C}\) denotes the mapping from the normalized image embedding to downstream class logits. The predicted class is then obtained by selecting the largest logit in \(M_y\).

\section{Reparameterized Inter-Class Visual Reprogramming}
\label{sec:method}

\begin{figure*}[t]
    \centering
    \includegraphics[width=\linewidth]{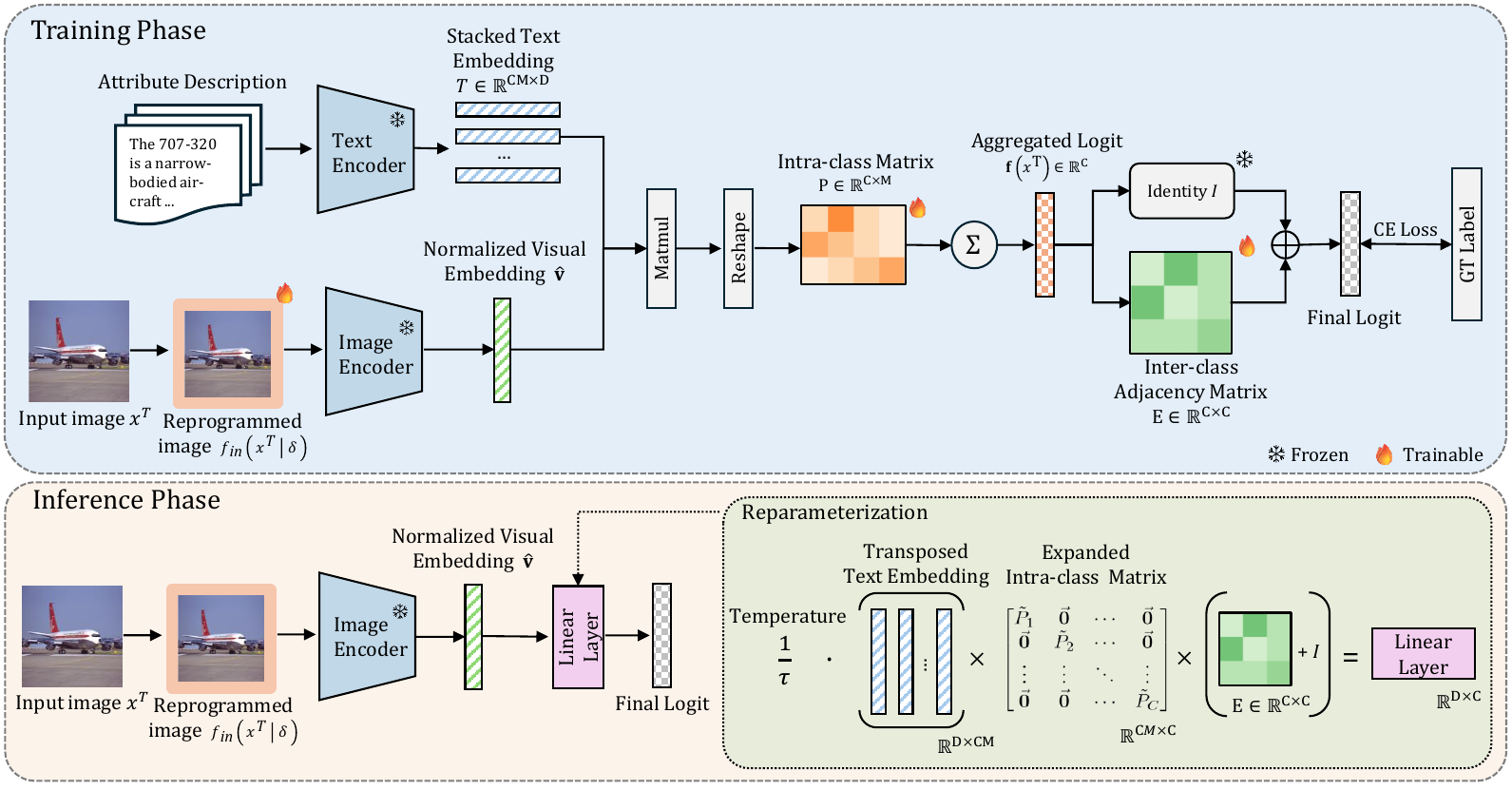}
    \vspace{-0.6cm}
    \caption{Overview of Reparameterized Inter-Class Visual Reprogramming (RVP). During training, a visual prompt transforms the input image, which is then encoded by a frozen image encoder. The cosine similarities between the visual embedding and text embeddings are aggregated by a learnable intra-class matrix to produce class logits, which are further refined by a residual class-relation matrix that captures inter-class dependencies. At inference, the text embeddings and label mapping are reparameterized into a single linear classifier, enabling efficient prediction with a single forward pass.}
    \vspace{-0.3cm}
    \label{fig:main}
\end{figure*}

Following recent advances~\citep{cai2025attribute,cai2025understanding}, we use multiple attribute-based textual descriptions to enrich the semantic space, with \(M\) descriptions for each of the \(C\) downstream classes. However, as discussed earlier, fine-grained categories often exhibit high similarity in the text embedding space, indicating strong inter-class correlations. Methods based solely on independent intra-class prompt selection cannot disentangle these shared dominant semantic directions. To address this limitation, we propose the Reparameterized Inter-Class Visual Reprogramming (RVP) framework.

\cref{fig:main}~illustrates the overall architecture of RVP. RVP jointly models intra-class description weighting and inter-class interactions while keeping the CLIP backbone frozen. The intra-class component refines class prototypes by learning how to aggregate attribute descriptions within each class, while the inter-class component enhances discriminative cues through message passing across classes. Instead of learning a dense \(CM \times C\) mapping from all descriptions to target classes, which introduces \(C^2M\) trainable parameters and is prone to overfitting in few-shot settings, RVP adopts an explicitly structured formulation with only \(C \times M\) parameters for intra-class weighting and \(C \times C\) parameters for inter-class modeling. This reduced parameterization imposes a useful structural prior and improves generalization. In addition, unlike Decoupled Visual Prompting~\citep{cai2025understanding}, which ensembles multiple visual prompts and therefore requires multiple forward passes through the image encoder, RVP uses a single visual prompt and only one forward pass. 

\subsection{Training Phase: Dynamic Aggregation and Message Passing}
Given a downstream image \(x^{\text{T}} \in \mathcal{X}^{\text{T}}\), we first apply a trainable visual transformation \(f_{\text{in}}(\cdot \mid \delta)\) to map it into CLIP's input space, where the transformation is implemented as resizing and boundary padding parameterized by a visual prompt \(\delta\)~\citep{cai2024sample}. The reprogrammed image is then fed into the frozen CLIP image encoder to obtain the \(\ell_2\)-normalized feature vector \(\hat{\mathbf{v}}\).
Let $T \in \mathbb{R}^{CM \times D}$ be the stacked matrix of normalized text embeddings, assuming $M$ textual descriptions for each of the $C$ downstream classes. To refine intra-class prompt selection, we introduce a learnable intra-class weighting matrix \(P \in \mathbb{R}^{C \times M}\). We apply a softmax over the description dimension to obtain normalized weights
$\tilde{P}_c = \operatorname{softmax}(P_c),$
where \(P_c\) denotes the \(c\)-th row of \(P\).
The aggregated base logit for class \(c\) is computed as
\begin{equation}
    f_c(x^{\mathrm{T}})
    =
    \sum_{m=1}^{M}
    \tilde{P}_{c,m}
    \frac{1}{\tau}\hat{\mathbf{v}}^\top \hat{\mathbf{t}}_{c,m},
\end{equation}
where \(\hat{\mathbf{t}}_{c,m} \in \mathbb{R}^{D}\) is the normalized embedding of the \(m\)-th description for class \(c\), and \(\tilde{P}_{c,m}\) is its corresponding normalized weight. 

\paragraph{How to Model Inter-class Relationships.}

In fine-grained tasks, textual descriptions from different classes are often highly similar and may share dominant semantic directions. As a result, even after aggregating multiple descriptions within each class, the base logits are still formed independently across classes and can retain strong cross-class ambiguity: semantically related but incorrect classes may receive high scores because shared text semantics are not explicitly suppressed. 

Let $\mathbf{f}(x^{\mathrm{T}})
    =
    [f_1(x^{\mathrm{T}}), \dots, f_C(x^{\mathrm{T}})]
    \in \mathbb{R}^{C}$
denote the row vector of base logits.
To model these dependencies, we introduce an inter-class adjacency matrix \(E \in \mathbb{R}^{C \times C}\) and formulate the final class logits \(\mathbf{z} \in \mathbb{R}^{C}\) as a residual graph message-passing step:
\begin{equation}
    \label{eq:message_passing}
    \mathbf{z} = \mathbf{f}(x^{\text{T}}) + \mathbf{f}(x^{\text{T}}) E = \mathbf{f}(x^{\text{T}}) (I + E).
\end{equation}
Here, the term \(\mathbf{f}(x^{\text{T}})E\) models how evidence should be redistributed across correlated classes, allowing RVP to suppress shared semantic components and enhance subtle class-specific differences that cannot be recovered from independent intra-class aggregation alone.

The residual formulation in~\eqref{eq:message_passing} provides an important structural prior. Instead of learning a dense unconstrained transformation from scratch, the identity matrix \(I\) preserves the original class logits induced by CLIP's pretrained alignment, while the matrix \(E\), initialized to zero, learns only residual corrections between classes. This design constrains the model to refine, rather than overwrite, the pretrained semantic structure, making optimization easier and more stable in the few-shot regime. As a result, it reduces the effective complexity of the mapping and helps the model focus on the subtle discriminative differences among correlated classes.

All trainable components in RVP are learned jointly with the downstream cross-entropy loss. Specifically, the visual prompt \(\delta\), intra-class matrix \(P\), and inter-class matrix \(E\) are optimized end-to-end, while the CLIP image and text encoders remain frozen. The final logits \(\mathbf{z}_i\) are supervised by
\begin{equation}
\mathcal{L}_{\mathrm{CE}}(\delta,P,E)
=
-\frac{1}{N}
\sum_{i=1}^{N}
\log
\frac{\exp(z_{i,y_i^{T}})}
{\sum_{c=1}^{C}\exp(z_{i,c})},
\end{equation}
where \(z_{i,c}\) is the logit of class \(c\) for sample \(i\), and \(z_{i,y_i^{T}}\) is the logit of its ground-truth class.

\subsection{Inference Phase: Linear Reparameterization}
During inference, the intra-class aggregation and inter-class message-passing can be completely absorbed into a single projection matrix. We construct a sparse routing matrix $W_1 \in \mathbb{R}^{CM \times C}$ as a block-diagonal matrix:
\begin{equation}
\label{eq:w1_structure}
    W_1 = 
    \begin{bmatrix}
        \tilde{P}_1 & \vec{\mathbf{0}} & \cdots & \vec{\mathbf{0}} \\
        \vec{\mathbf{0}} & \tilde{P}_2 & \cdots & \vec{\mathbf{0}} \\
        \vdots & \vdots & \ddots & \vdots \\
        \vec{\mathbf{0}} & \vec{\mathbf{0}} & \cdots & \tilde{P}_C
    \end{bmatrix} \in \mathbb{R}^{CM \times C},
\end{equation}
where each $\tilde{P}_c = [\tilde{P}_{c,1}, \dots, \tilde{P}_{c,M}]^\top \in \mathbb{R}^M$ is the normalized weight vector for the $M$ descriptions of class $c$, and $\vec{\mathbf{0}}$ denotes a zero column vector of length $M$. By leveraging the associativity of matrix multiplication, we pre-calculate a unified classifier matrix $\hat{W} \in \mathbb{R}^{D \times C}$ that absorbs the text embeddings, routing weights, and inter-class correlations:
\begin{equation}
\label{eq:reparameterized_classifier}
    \hat{W} = \frac{1}{\tau} T^\top W_1 (I + E).
\end{equation}
The detailed structure of this reparameterized head can be unrolled as:
\begin{equation}
\label{eq:matrix_shape_final}
    \hat{W} = \frac{1}{\tau} 
    \underbrace{
    \begin{pmatrix} 
        \mid & \mid & & \mid \\
        \hat{\mathbf{t}}_{1,1} & \hat{\mathbf{t}}_{1,2} & \cdots & \hat{\mathbf{t}}_{C,M} \\
        \mid & \mid & & \mid 
    \end{pmatrix}
    }_{T^\top \in \mathbb{R}^{D \times CM}}
    \underbrace{
    \begin{pmatrix}
        \tilde{P}_1 & \vec{\mathbf{0}} & \cdots \\
        \vec{\mathbf{0}} & \tilde{P}_2 & \cdots \\
        \vdots & \vdots & \ddots 
    \end{pmatrix}
    }_{W_1 \in \mathbb{R}^{CM \times C}}
    \underbrace{
    \vphantom{\begin{pmatrix} \mid \\ \mathbf{t} \\ \mid \end{pmatrix}}
    (I + E)
    }_{\mathbb{R}^{C \times C}}.
\end{equation}
Consequently, the entire forward pass during inference is reduced to generating the single visual prompt, extracting the normalized image embedding via exactly one pass through the backbone, and performing a single linear projection: $\mathbf{z} = \hat{\mathbf{v}}^\top \hat{W}.$

This exact reparameterization replaces the standard aggregation weights with $\hat{W}$, guaranteeing that RVP identifies a more expressive mapping $\phi$ while adding nearly zero overhead during inference. 


\section{Experiments}
\label{sec:exp}

\paragraph{Experimental Setup and Benchmarks.}
To evaluate the proposed RVP framework, we adhere to the established experimental protocol from \citep{cai2025understanding}. We conduct all experiments using pretrained CLIP models across four image encoder architectures, including variants of ResNet~\citep{resnet} and Vision Transformer (ViT)~\citep{dosovitskiy2020vit}, under a 16-shot downstream classification setting. All reported results represent the average accuracy across three independent random seeds. Our benchmark suite comprises 11 datasets covering diverse visual domains, including textures, actions, and natural scenes. All datasets are publicly available: FGVC Aircraft (Aircraft)~\citep{aircraft}, Caltech101 (Caltech)~\citep{caltech101}, StanfordCars (Cars)~\citep{stanfordcars}, Describable Textures Dataset (DTD)~\citep{dtd}, EuroSAT (ESAT)~\citep{eurosat}, Flowers102 (Flowers)~\citep{flowers}, Food101 (Food)~\citep{food101}, OxfordPets (Pets)~\citep{parkhi2012cats}, SUN397 (SUN)~\citep{sun397}, UCF101 (UCF)~\citep{ucf101}, and RESISC45 (Resisc)~\citep{resisc}. More implementation details can be found in \cref{sec:implementation}.

\section{Results}
\label{sec:results}

\paragraph{Quantitative Results.}

We compare RVP against four prominent visual reprogramming (VR) baselines: 
(1) \textbf{VP} \citep{bahng2022exploring}, a standard VR approach that overlays learnable pixel perturbations onto rescaled downstream images; 
(2) \textbf{AR} \citep{tsai2020transfer,chen2023understanding}, which pads learnable noise parameters around the image boundary; 
(3) \textbf{AttrVR} \citep{cai2025attribute}, which guides the learning of visual prompt patterns using class-specific attribute descriptions; and 
(4) \textbf{DVP} \citep{cai2025understanding}, a decoupled visual prompting framework that ensembles multiple reprogrammed inputs. For a strictly fair comparison, we evaluate the unsupervised clustering variant of DVP (DVP-cls), which isolates the performance of the reprogramming mechanism without relying on external Large Language Models (LLMs) to generate cause-specific descriptions. RVP uses the same text prompt set as DVP.

As shown in \cref{tab:mainres,tab:rn50res,tab:backbone}, RVP achieves the highest average accuracy across all evaluated backbones. With ViT-B/16 CLIP (\cref{tab:mainres}), RVP attains the highest average accuracy of \(82.7\%\), outperforming DVP~\citep{cai2025understanding} by \(3.0\) points and AttrVR~\citep{cai2025attribute} by \(4.2\) points. It achieves the best result on 9 out of 11 datasets, with especially large gains on fine-grained benchmarks such as Aircraft (\(+7.4\) over DVP) and Cars (\(+14.0\)). These improvements are consistent with our motivation: in fine-grained recognition, many categories share highly similar semantic attributes, and the main discriminative cues arise from subtle inter-class differences. In this regime, intra-class prompt selection alone is insufficient, while RVP can explicitly suppress shared semantic components and amplify class-specific distinctions through inter-class modeling.

The advantage of RVP becomes even more pronounced with the weaker RN50 backbone. In \cref{tab:rn50res}, RVP achieves an average accuracy of \(72.3\%\), exceeding DVP by \(6.3\) points and AttrVR by \(7.7\) points. It again shows particularly strong gains on Aircraft and Cars, indicating that the proposed structured reparameterization remains effective even when the underlying visual encoder is less expressive.

\begin{table*}[t]
    \centering
    \caption{Accuracy comparison of different methods trained on 16-shot downstream classification tasks, using ViT-B/16-based CLIP as the pretrained model (Mean \% ± Std \%, ours are \textcolor{black}{\colorbox{gray!30}{highlighted}} and the highest result is in \textbf{bold}). Other results are taken from prior work~\citep{cai2025understanding}.}
    \begin{sc}
    \resizebox{\textwidth}{!}{
    \begin{tabular}{c|ccccccccccc|c}
    \toprule
    Method                & Aircraft      & Caltech       & Cars          & DTD           & ESAT          & Flowers       & Food          & Pets          & SUN           & UCF     & Resisc        & Avg.                           \\
    \midrule
     VP    & 32.1         & 93.5          & 65.5          & 61.4          & 91.2          & 82.5          & 82.3          & 91.0          & 65.8          & 73.8                & 79.1          & {74.4}          \\
    AR    & 31.7          & 95.5           & 68.0           & 62.0           & 93.4          & 85.9           & 85.2           & 92.7           & 67.9           & 78.1                 & 81.6           & {76.5}          \\
    AttrVR & {36.6}  & {95.7} & {68.3} & {65.6} & {93.8} & {92.9} & \textbf{85.9} & {93.3} & {69.6} & {79.0}  & {82.6} & {78.5} \\
    DVP & 38.7 & 96.0 & 70.8 & 65.5 & \textbf{94.1} & 95.0 & 85.7 & {93.3} & {71.1} & {82.0} & 84.4 & 79.7 \\
    \rowcolor{gray!30}
    RVP &
\textbf{46.1}\scriptsize{$\pm$0.2} & \textbf{96.5}\scriptsize{$\pm$0.2} & \textbf{84.8}\scriptsize{$\pm$0.3} & \textbf{68.7}\scriptsize{$\pm$0.1} & 92.7\scriptsize{$\pm$0.2} & \textbf{96.7}\scriptsize{$\pm$0.2} & 85.5\scriptsize{$\pm$0.1} & \textbf{94.0}\scriptsize{$\pm$0.1} & \textbf{73.8}\scriptsize{$\pm$0.1} & \textbf{85.1}\scriptsize{$\pm$0.7} & \textbf{85.9}\scriptsize{$\pm$0.5} & \textbf{82.7} \\
    \bottomrule
    \end{tabular}}
    \end{sc}
    \label{tab:mainres}
    \vspace{-0.1cm}
\end{table*}

\begin{table*}[!t]
    \centering
\caption{Accuracy comparison of different methods trained on 16-shot downstream classification tasks, using RN50-based CLIP as the pretrained model (Mean \% ± Std \%, ours are \textcolor{black}{\colorbox{gray!30}{highlighted}} and the highest result is in \textbf{bold}). Other results are taken from prior work~\citep{cai2025understanding}.}
\begin{sc}
\resizebox{\textwidth}{!}{
    \begin{tabular}{c|ccccccccccc|c}
    \toprule
    Method             & Aircraft      & Caltech       & Cars          & DTD           & ESAT          & Flowers       & Food          & Pets          & SUN           & UCF     & Resisc        & Avg.                           \\
    \midrule
 VP  & 16.2          & 80.1          & 44.0          & 43.4          & 59.7          & 53.6          & 65.3          & 77.2          & 48.8          & 52.0          & 47.7          & 53.5 \\
 AR    & 18.6          & 86.5          & 53.9          & 46.4          & 66.6          & 60.9          & 74.2          & 82.5          & 56.8          & 59.7          & 58.4          & 60.4 \\
AttrVR & 20.7          & 89.1          & 53.9          & 54.4          & 72.0          & 74.8          & \textbf{75.3} & 88.9          & 59.9          & 63.6          & 58.2          & 64.6 \\
DVP & 22.1          & 89.8          & 54.5          & 55.9          & 72.2          & 80.0          & 75.0          & {88.9} & {61.1} & {65.9} & {60.8} & 66.0 \\
\rowcolor{gray!30}
RVP & \textbf{29.1}\scriptsize{$\pm$0.3} & \textbf{92.0}\scriptsize{$\pm$0.0} & \textbf{71.0}\scriptsize{$\pm$0.3} & \textbf{62.0}\scriptsize{$\pm$0.3} & \textbf{72.9}\scriptsize{$\pm$0.9} & \textbf{91.6}\scriptsize{$\pm$0.2} & {73.7}\scriptsize{$\pm$0.0} & \textbf{89.7}\scriptsize{$\pm$0.1} & \textbf{66.2}\scriptsize{$\pm$0.1} & \textbf{74.3}\scriptsize{$\pm$0.1} & \textbf{72.9}\scriptsize{$\pm$0.1} & \textbf{72.3} \\
\bottomrule
\end{tabular}}
\end{sc}
\label{tab:rn50res}
\vspace{-0.3cm}
\end{table*}

\begin{wraptable}{r}{0.50\linewidth}
\vspace{-0.3cm}
\centering
\footnotesize
\caption{Average accuracy of different VR methods on 11 datasets using different CLIP visual encoders (mean accuracy in \%; ours are \textcolor{black}{\colorbox{gray!30}{highlighted}} and the highest is in \textbf{bold}; RN denotes ResNet).}
\vspace{-0.3cm}
\label{tab:backbone}
\begin{sc}
\setlength{\tabcolsep}{4pt}
\begin{tabular}{l|cccc}
\toprule
Method & RN50 & RN101 & ViT-B/32 & ViT-B/16 \\
\midrule
VP      & 53.5 & 57.5 & 68.3 & 74.4 \\
AR      & 60.4 & 62.7 & 66.3 & 76.5 \\
AttrVR  & 64.6 & 67.2 & 69.8 & 78.5 \\
DVP     & 66.0 & 68.8 & 71.0 & 79.7 \\
\rowcolor{gray!30}
RVP     & \textbf{72.3} & \textbf{73.0} & \textbf{75.8} & \textbf{82.7} \\ 
\bottomrule
\end{tabular}
\end{sc}
\vspace{-0.2cm}
\end{wraptable}

More broadly, \cref{tab:backbone} shows that the advantage of RVP becomes larger as the visual backbone becomes weaker. Compared with DVP, the gain is \(+6.3\) on RN50, \(+4.2\) on RN101, \(+4.8\) on ViT-B/32, and \(+3.0\) on ViT-B/16. This trend is expected, since weaker pretrained encoders produce less separable features, making downstream classification more reliant on the quality of the label mapping. In this setting, a structured mapping that explicitly models intra-class aggregation and inter-class relationships becomes more effective, as it can recover discriminative information that is not well separated in the original feature space. In contrast, stronger backbones such as ViT-B/16 already provide more discriminative and semantically aligned embeddings, leaving less room for improvement.

\paragraph{Few-shot Classification Performance.}

\begin{wrapfigure}{r}{0.45\textwidth}
    \centering
    \vspace{-0.5cm}
    \includegraphics[width=0.45\textwidth]{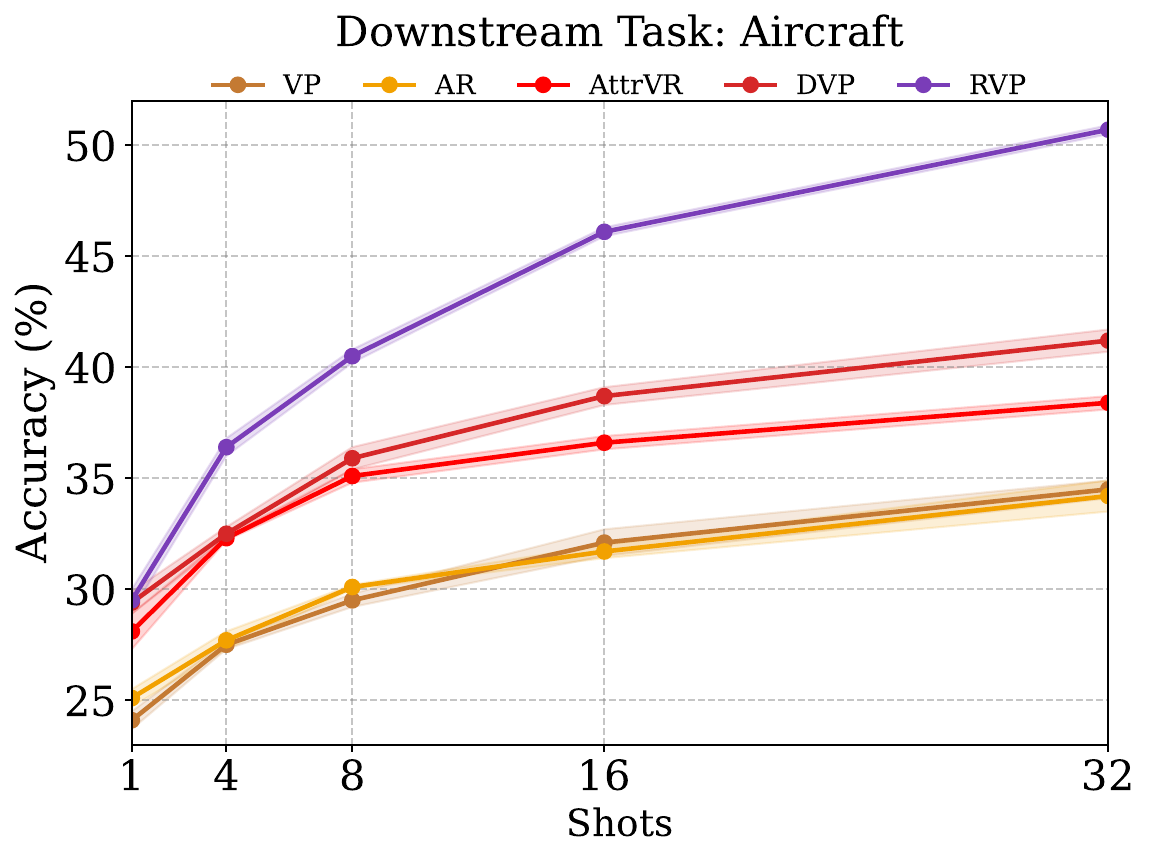}
    \vspace{-0.7cm}
    \caption{Accuracy comparison across different shot settings on Aircraft using ViT-B/16 CLIP. RVP consistently outperforms prior VR methods across all shot numbers. Shaded regions indicate standard deviation.}
    \vspace{-0.6cm}
    \label{fig:shots-aircraft}
\end{wrapfigure}
\cref{fig:shots-aircraft} compares few-shot performance on the Aircraft dataset under 1, 4, 8, 16, and 32 training samples per class. RVP achieves the best accuracy at every shot setting and shows a clear advantage over all prior visual reprogramming baselines. The improvement is modest in the extreme 1-shot setting, but becomes much larger as more labeled samples are available. At 32-shot, RVP attains 50.7\%, substantially outperforming DVP (41.2\%) and AttrVR (38.4\%).

This trend suggests that the proposed structured mapping can make better use of additional supervision than existing methods. While all approaches improve as the number of shots increases, the gain of RVP is much steeper, indicating stronger scalability from low-shot to moderately supervised settings. The relatively small standard deviations across all shot numbers also show that the improvements are stable over repeated runs. Overall, the figure shows that explicitly modeling both intra-class aggregation and inter-class relationships leads to more effective few-shot adaptation on fine-grained recognition tasks.

\paragraph{Accuracy and Inference Speed.}
RVP not only improves classification accuracy, but also retains the efficiency advantage of standard visual reprogramming. As shown in \cref{fig:rvp_bar}, RVP consistently achieves the highest accuracy on FGVC Aircraft across all four backbones, while maintaining latency close to single-prompt methods such as VP, AR, and AttrVR. In contrast, DVP relies on multiple decoupled visual prompts and therefore requires multiple forward passes through the frozen image encoder, leading to substantially higher inference latency. By using only a single visual prompt and reparameterizing the output mapping into one linear classifier, RVP avoids this overhead while still delivering stronger performance. This result shows that the gain of RVP comes from a more effective structured label mapping rather than increased inference-time computation.

\begin{figure*}[t]
    \centering
    \includegraphics[width=\linewidth]{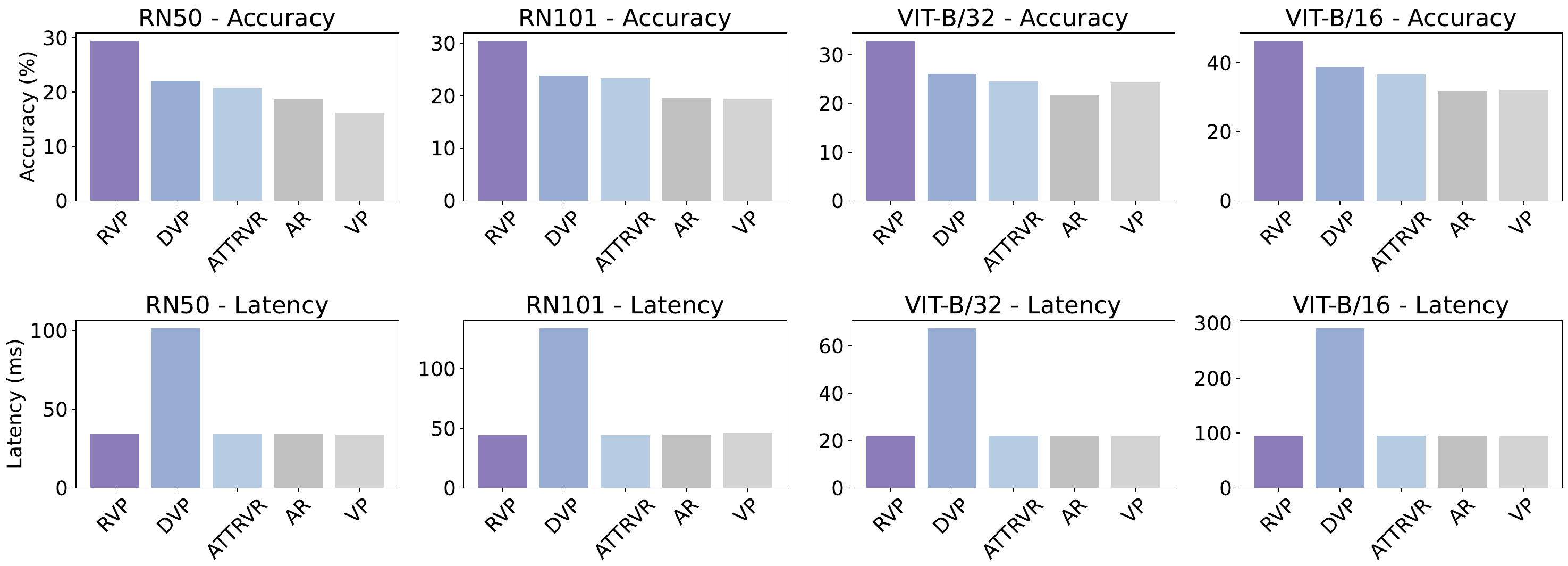}
    \vspace{-0.5cm}
    \caption{Accuracy and latency comparison on FGVC Aircraft across different backbones. Our method consistently achieves the highest accuracy while maintaining low latency, demonstrating a favorable trade-off between performance and efficiency.}
    \vspace{-0.2cm}
    \label{fig:rvp_bar}
\end{figure*}

\begin{table*}[!t]
\centering
\caption{Ablation studies of RVP using a ViT-B/16-based CLIP backbone.  The complete method is \textcolor{black}{\colorbox{gray!30}{highlighted}}, and the best results are shown in \textbf{bold}.}
\begin{sc}
\resizebox{\textwidth}{!}{
\begin{tabular}{c|ccccccccccc|c}
\toprule
Method                & Aircraft      & Caltech       & Cars          & DTD           & ESAT          & Flowers       & Food          & Pets          & SUN           & UCF     & Resisc        & Avg.                           \\
\midrule
\rowcolor{gray!30}
RVP & \textbf{46.1} & \textbf{96.5} & \textbf{84.8} & \textbf{68.7} & 92.7 & \textbf{96.7} & 85.5 & \textbf{94.0} & \textbf{73.8} & \textbf{85.1} & \textbf{85.9} & \textbf{82.7} \\
w/o VR & 40.1 & 96.2 & 81.7 & 65.7 & 58.5 & 96.8 & 84.7 & 93.9 & 74.2 & 83.6 & 82.8 & 78.0 \\
w/o intra-class $P$ & 46.0 & 96.3 & 84.7 & 68.3 & 92.6 & 96.8 & 85.5 & 94.1 & 73.8 & 84.4 & 85.6 & 82.6 \\
w/o inter-class $E$ & 35.6 & 96.1 & 68.0 & 63.3 & \textbf{93.8} & 91.7 & \textbf{85.6} & 93.1 & 67.4 & 78.9 & 83.5 & 77.9 \\ 
Linear Probe & 37.2 & 93.9 & 73.5 & 63.5 & 84.2 & 92.2 & 79.5 & 85.6 & 69.1 & 76.8 & 83.5 & 76.3 \\
Attribute \& LP & 37.4 & 90.9 & 78.1 & 56.4 & 59.4 & 86.2 & 73.9 & 90.7 & 67.9 & 68.7 & 76.6 & 71.5 \\
\bottomrule
\end{tabular}}
\end{sc}
\label{tab:ablation_vitb16}
\vspace{-0.3cm}
\end{table*}

\paragraph{Ablation Study.}

We conduct an ablation study in \cref{tab:ablation_vitb16} using a ViT-B/16-based CLIP backbone. In addition to the full RVP model, we evaluate five variants: w/o VR, which removes visual reprogramming and classifies zero-padded images only; w/o intra-class \(P\), which replaces the learnable intra-class weights with uniform averaging; w/o inter-class \(E\), which removes inter-class modeling; Linear Probe, which trains a linear classifier on the visual embedding; and Attribute \& LP, which learns a dense \(CM \times C\) prompt-reweighting matrix over all text descriptions. Linear Probe and Attribute \& LP are trained together with the visual prompt.

The complete RVP achieves the best average accuracy of \(82.7\%\), confirming that visual reprogramming, intra-class aggregation, and inter-class modeling work best together. Removing VR reduces the average accuracy to \(78.0\%\), with large drops on Aircraft, Cars, DTD, and especially ESAT, showing that input adaptation remains important. 
Removing the inter-class matrix \(E\) reduces the average accuracy to \(77.9\%\), which is the largest drop among all architectural ablations. The effect is especially strong on Aircraft and Cars, where classes share many attributes and differ only in subtle details. This shows that explicit inter-class modeling is the main source of improvement in RVP.

The comparisons with Linear Probe and Attribute \& LP further highlight the value of structured reparameterization. Linear Probe reaches \(76.3\%\), while Attribute \& LP performs even worse at \(71.5\%\), despite using a dense mapping with more parameters. This indicates that simply increasing trainable parameter size can lead to overfitting without structural constraints.

Overall, the ablation results show a clear pattern: the inter-class module \(E\) is the most important component, the visual prompt is also essential, and the intra-class matrix \(P\) provides a smaller but consistent gain. These results support the design of RVP and show that its advantage comes from structured modeling rather than a larger unconstrained classifier.

\paragraph{Qualitative Analysis.}

\begin{figure*}[t]
    \centering
    \includegraphics[width=\linewidth]{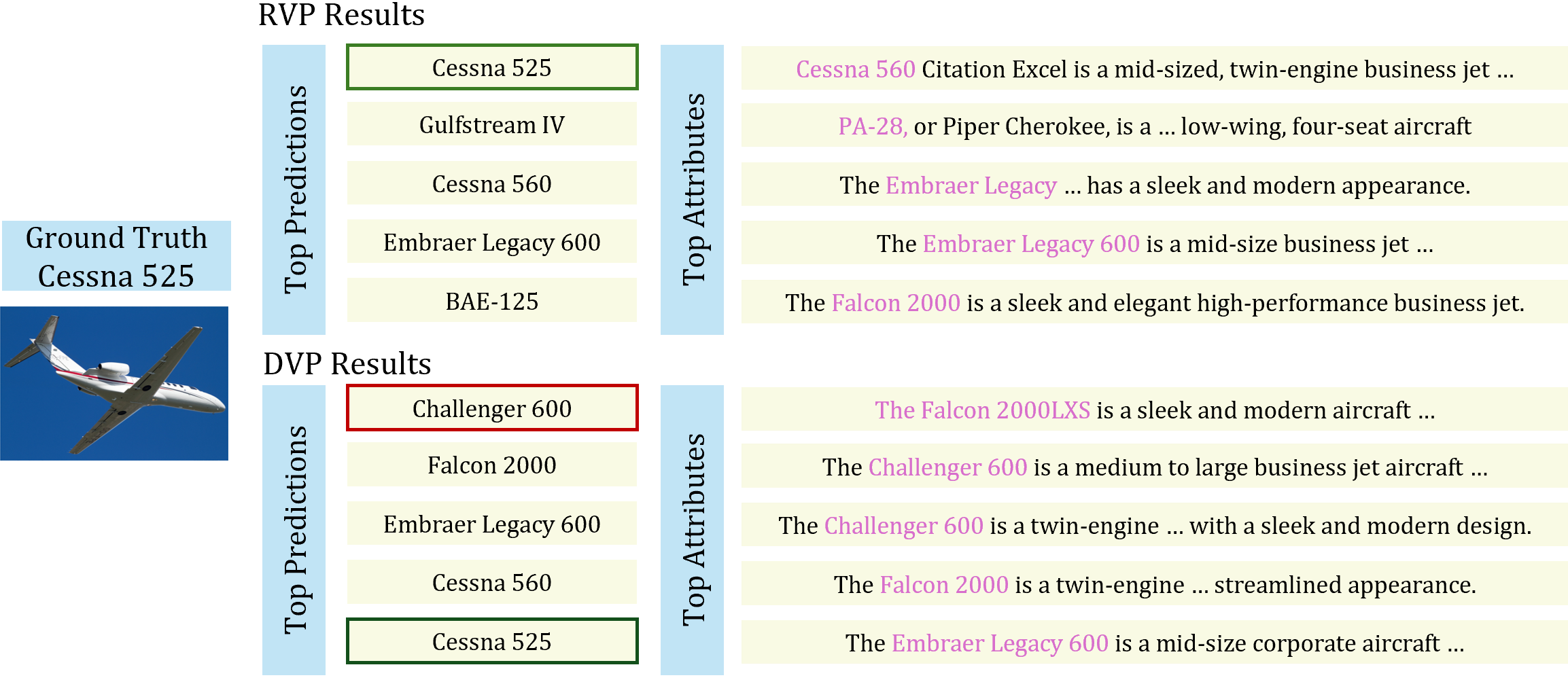}
    \caption{Top predicted classes and highest-matching attributes for a test image. For both RVP and DVP, the most similar attributes include prompts from other classes, reflecting strong semantic overlap in fine-grained recognition. However, RVP explicitly models inter-class relationships, allowing it to better resolve these cross-class ambiguities and produce the correct prediction.}
    \label{fig:vis1}
    \vspace{-0.4cm}
\end{figure*}

\cref{fig:vis1} provides a qualitative comparison between RVP and DVP on a fine-grained aircraft example. In both methods, the highest-matching attributes include prompts from semantically similar but incorrect classes. This is expected because attribute prompts are largely class-agnostic: descriptions such as \textit{sleek}, \textit{streamlined}, or \textit{twin-engine} are often shared by multiple aircraft categories and therefore cannot uniquely identify a class on their own. As a result, relying only on prompt-level similarity can lead to ambiguous predictions. RVP addresses this issue by explicitly modeling inter-class relationships, allowing it to suppress misleading shared semantics and rank the ground-truth class Cessna 525 at the top, whereas DVP places it only in the fifth position.



\section{Limitations}
\label{sec:limitation}
Although RVP performs strongly on most benchmarks, its performance is less pronounced on Food101. Unlike fine-grained object categories, food images often include not only the main dish but also side dishes, garnish, sauces, and other accompanying ingredients, which makes class semantics less stable across samples. This high intra-class variation weakens the consistency of text-based class relationships and reduces the advantage of our structured inter-class modeling. A more detailed discussion is provided in \cref{app:limitation}.

\section{Conclusion}

We proposed Reparameterized Inter-Class Visual Reprogramming (RVP), which adapts frozen vision--language models through structured intra-class aggregation and inter-class modeling. RVP admits exact reparameterization into a single linear classifier at inference. Across 11 benchmarks and multiple CLIP backbones, RVP consistently improves over prior visual reprogramming methods, with the largest gains on fine-grained tasks. These results highlight the benefit of modeling class relationships in visual reprogramming and motivate extending RVP beyond CLIP. We will explore the application of RVP beyond CLIP in future work.


\bibliography{iclr2027_conference}

@inproceedings{chen2023understanding,
  title     = {Understanding and improving visual prompting: A label-mapping perspective},
  author    = {Chen, Aochuan and Yao, Yuguang and Chen, Pin-Yu and Zhang, Yihua and Liu, Sijia},
  booktitle = {CVPR},
  year      = {2023}
}

@inproceedings{tsao2023autovp,
  title     = {Autovp: An automated visual prompting framework and benchmark},
  author    = {Tsao, Hsi-Ai and Hsiung, Lei and Chen, Pin-Yu and Liu, Sijia and Ho, Tsung-Yi},
  booktitle = {ICLR},
  year      = {2024}
}

@inproceedings{tsai2020transfer,
  title     = {Transfer Learning without Knowing: Reprogramming Black-box Machine Learning Models with Scarce Data and Limited Resources},
  author    = {Tsai, Yun-Yun and Chen, Pin-Yu and Ho, Tsung-Yi},
  booktitle = {ICML},
  year      = {2020}
}

@inproceedings{chen2024model,
  title     = {Model reprogramming: Resource-efficient cross-domain machine learning},
  author    = {Chen, Pin-Yu},
  booktitle = {AAAI},
  year      = {2024}
}

@inproceedings{cai2024sample,
  title     = {Sample-specific Masks for Visual Reprogramming-based Prompting},
  author    = {Cai, Chengyi and Ye, Zesheng and Feng, Lei and Qi, Jianzhong and Liu, Feng},
  booktitle = {ICML},
  year      = {2024}
}

@inproceedings{cai2024bayesian,
  title     = {Bayesian-guided Label Mapping for Visual Reprogramming},
  author    = {Cai, Chengyi and Ye, Zesheng and Feng, Lei and Qi, Jianzhong and Liu, Feng},
  booktitle = {NeurIPS},
  year      = {2024}
}

@inproceedings{elsayedadversarial,
  title     = {Adversarial Reprogramming of Neural Networks},
  author    = {Elsayed, Gamaleldin F and Goodfellow, Ian and Sohl-Dickstein, Jascha},
  booktitle = {ICLR},
  year      = {2018}
}

@inproceedings{oh2023blackvip,
  title     = {Blackvip: Black-box visual prompting for robust transfer learning},
  author    = {Oh, Changdae and Hwang, Hyeji and Lee, Hee-young and Lim, YongTaek and Jung, Geunyoung and Jung, Jiyoung and Choi, Hosik and Song, Kyungwoo},
  booktitle = {CVPR},
  year      = {2023}
}

@article{bahng2022exploring,
  title   = {Exploring visual prompts for adapting large-scale models},
  author  = {Bahng, Hyojin and Jahanian, Ali and Sankaranarayanan, Swami and Isola, Phillip},
  journal = {arXiv},
  year    = {2022}
}

@article{aircraft,
  title   = {Fine-grained visual classification of aircraft},
  author  = {Maji, Subhransu and Rahtu, Esa and Kannala, Juho and Blaschko, Matthew and Vedaldi, Andrea},
  journal = {arXiv},
  year    = {2013}
}

@inproceedings{caltech101,
  title     = {Learning generative visual models from few training examples: An incremental bayesian approach tested on 101 object categories},
  author    = {Fei-Fei, Li and Fergus, Rob and Perona, Pietro},
  booktitle = {CVPR workshop},
  year      = {2004}
}

@inproceedings{stanfordcars,
  title     = {3d object representations for fine-grained categorization},
  author    = {Krause, Jonathan and Stark, Michael and Deng, Jia and Fei-Fei, Li},
  booktitle = {ICCV workshops},
  year      = {2013}
}

@inproceedings{flowers,
  title     = {Automated flower classification over a large number of classes},
  author    = {Nilsback, Maria-Elena and Zisserman, Andrew},
  booktitle = {Indian Conference on Computer Vision, Graphics \& Image Processing},
  year      = {2008}
}

@inproceedings{food101,
  title     = {Food-101--mining discriminative components with random forests},
  author    = {Bossard, Lukas and Guillaumin, Matthieu and Van Gool, Luc},
  booktitle = {ECCV},
  year      = {2014}
}

@inproceedings{sun397,
  title     = {Sun database: Large-scale scene recognition from abbey to zoo},
  author    = {Xiao, Jianxiong and Hays, James and Ehinger, Krista A and Oliva, Aude and Torralba, Antonio},
  booktitle = {CVPR},
  year      = {2010}
}

@inproceedings{dtd,
  title     = {Describing textures in the wild},
  author    = {Cimpoi, Mircea and Maji, Subhransu and Kokkinos, Iasonas and Mohamed, Sammy and Vedaldi, Andrea},
  booktitle = {CVPR},
  year      = {2014}
}

@article{eurosat,
  title   = {Eurosat: A novel dataset and deep learning benchmark for land use and land cover classification},
  author  = {Helber, Patrick and Bischke, Benjamin and Dengel, Andreas and Borth, Damian},
  journal = {IEEE Journal of Selected Topics in Applied Earth Observations and Remote Sensing},
  year    = {2019}
}

@article{resisc,
  title   = {Remote sensing image scene classification: Benchmark and state of the art},
  author  = {Cheng, Gong and Han, Junwei and Lu, Xiaoqiang},
  journal = {Proceedings of the IEEE},
  year    = {2017}
}

@article{ucf101,
  title   = {A dataset of 101 human action classes from videos in the wild},
  author  = {Soomro, Khurram and Zamir, Amir Roshan and Shah, Mubarak},
  journal = {Center for Research in Computer Vision},
  year    = {2012}
}

@article{kirkpatrick2017overcoming,
  title   = {Overcoming catastrophic forgetting in neural networks},
  author  = {Kirkpatrick, James and Pascanu, Razvan and Rabinowitz, Neil and Veness, Joel and Desjardins, Guillaume and Rusu, Andrei A and Milan, Kieran and Quan, John and Ramalho, Tiago and Grabska-Barwinska, Agnieszka and others},
  journal = {PNAS},
  year    = {2017}
}

@inproceedings{hambardzumyan2021warp,
  title     = {WARP: Word-level adversarial reprogramming},
  author    = {Hambardzumyan, K and Khachatrian, H and May, J},
  booktitle = {ACL-IJCNLP},
  year      = {2021}
}

@inproceedings{vinod2020reprogramming,
  title     = {Reprogramming Language Models for Molecular Representation Learning},
  author    = {Vinod, Ria and Chen, Pin-Yu and Das, Payel},
  booktitle = {NeurIPS},
  year      = {2020}
}

@inproceedings{hung2023low,
  title     = {Low-resource music genre classification with cross-modal neural model reprogramming},
  author    = {Hung, Yun-Ning and Yang, Chao-Han Huck and Chen, Pin-Yu and Lerch, Alexander},
  booktitle = {ICASSP},
  year      = {2023}
}

@inproceedings{yang2023english,
  title     = {From english to more languages: Parameter-efficient model reprogramming for cross-lingual speech recognition},
  author    = {Yang, Chao-Han Huck and Li, Bo and Zhang, Yu and Chen, Nanxin and Prabhavalkar, Rohit and Sainath, Tara N and Strohman, Trevor},
  booktitle = {ICASSP},
  year      = {2023}
}

@inproceedings{yang2021voice2series,
  title     = {Voice2series: Reprogramming acoustic models for time series classification},
  author    = {Yang, Chao-Han Huck and Tsai, Yun-Yun and Chen, Pin-Yu},
  booktitle = {ICML},
  year      = {2021}
}

@inproceedings{jing2023deep,
  title     = {Deep graph reprogramming},
  author    = {Jing, Yongcheng and Yuan, Chongbin and Ju, Li and Yang, Yiding and Wang, Xinchao and Tao, Dacheng},
  booktitle = {CVPR},
  year      = {2023}
}

@inproceedings{zhou2022conditional,
  title     = {Conditional prompt learning for vision-language models},
  author    = {Zhou, Kaiyang and Yang, Jingkang and Loy, Chen Change and Liu, Ziwei},
  booktitle = {CVPR},
  year      = {2022}
}

@article{zhou2022learning,
  title   = {Learning to prompt for vision-language models},
  author  = {Zhou, Kaiyang and Yang, Jingkang and Loy, Chen Change and Liu, Ziwei},
  journal = {IJCV},
  year    = {2022}
}

@inproceedings{khattak2023maple,
  title     = {Maple: Multi-modal prompt learning},
  author    = {Khattak, Muhammad Uzair and Rasheed, Hanoona and Maaz, Muhammad and Khan, Salman and Khan, Fahad Shahbaz},
  booktitle = {CVPR},
  year      = {2023}
}

@inproceedings{wang2023transhp,
  title     = {Transhp: Image classification with hierarchical prompting},
  author    = {Wang, Wenhao and Sun, Yifan and Li, Wei and Yang, Yi},
  booktitle = {NeurIPS},
  year      = {2023}
}

@inproceedings{huang2023diversity,
  title     = {Diversity-aware meta visual prompting},
  author    = {Huang, Qidong and Dong, Xiaoyi and Chen, Dongdong and Zhang, Weiming and Wang, Feifei and Hua, Gang and Yu, Nenghai},
  booktitle = {CVPR},
  year      = {2023}
}

@inproceedings{radford2021learning,
  title     = {Learning transferable visual models from natural language supervision},
  author    = {Radford, Alec and Kim, Jong Wook and Hallacy, Chris and Ramesh, Aditya and Goh, Gabriel and Agarwal, Sandhini and Sastry, Girish and Askell, Amanda and Mishkin, Pamela and Clark, Jack and others},
  booktitle = {ICML},
  year      = {2021}
}

@inproceedings{parkhi2012cats,
  title     = {Cats and dogs},
  author    = {Parkhi, Omkar M and Vedaldi, Andrea and Zisserman, Andrew and Jawahar, CV},
  booktitle = {CVPR},
  year      = {2012}
}

@inproceedings{loshchilov2016sgdr,
title={{SGDR}: Stochastic Gradient Descent with Warm Restarts},
author={Ilya Loshchilov and Frank Hutter},
booktitle={International Conference on Learning Representations},
year={2017},
url={https://openreview.net/forum?id=Skq89Scxx}
}

@inproceedings{cai2025attribute,
title={Attribute-based Visual Reprogramming for Vision-Language Models},
author={Chengyi Cai and Zesheng Ye and Lei Feng and Jianzhong Qi and Feng Liu},
booktitle={The Thirteenth International Conference on Learning Representations},
year={2025},
url={https://openreview.net/forum?id=j964C6y92q}
}

@inproceedings{Zhang_2024_CVPR,
  author    = {Zhang, Yichi and Dong, Yinpeng and Zhang, Siyuan and Min, Tianzan and Su, Hang and Zhu, Jun},
  title     = {Exploring the Transferability of Visual Prompting for Multimodal Large Language Models},
  booktitle = {CVPR},
  year      = {2024}
}

@inproceedings{jin2025lorvp,
  title     = {LoR-{VP}: Low-Rank Visual Prompting for Efficient Vision Model Adaptation},
  author    = {Can Jin and Ying Li and Mingyu Zhao and Shiyu Zhao and Zhenting Wang and Xiaoxiao He and Ligong Han and Tong Che and Dimitris N. Metaxas},
  booktitle = {ICLR},
  year      = {2025}
}

@inproceedings{zheng2025endow,
  title     = {Endowing Visual Reprogramming with Adversarial Robustness},
  author    = {Shengjie Zhou and Xin Cheng and Haiyang Xu and Ming Yan and Tao Xiang and Feng Liu and Lei Feng},
  booktitle = {ICLR},
  year      = {2025}
}

@inproceedings{chen2025refine,
  title     = {REFINE: Inversion-Free Backdoor Defense via Model Reprogramming},
  author    = {Yukun Chen and Shuo Shao and Enhao Huang and Yiming Li and Pin-Yu Chen and Zhan Qin and Kui Ren},
  booktitle = {ICLR},
  year      = {2025}
}

@inproceedings{jin2024time,
  title     = {{Time-LLM}: Time series forecasting by reprogramming large language models},
  author    = {Ming Jin and Shiyu Wang and Lintao Ma and Zhixuan Chu and James Y. Zhang and Xiaoming Shi and Pin-Yu Chen and Yuxuan Liang and Yuan-Fang Li and Shirui Pan and Qingsong Wen},
  booktitle = {ICLR},
  year      = {2024}
}

@inproceedings{yen2023neural,
  title     = {Neural Model Reprogramming with Similarity Based Mapping for Low-Resource Spoken Command Classification},
  author    = {Hao Yen and Pin-Jui Ku and Chao-Han Huck Yang and Hu Hu and Sabato Marco Siniscalchi and Pin-Yu Chen and Yu Tsao},
  booktitle = {INTERSPEECH},
  year      = {2023}
}

@inproceedings{cai2025understanding,
    title={Understanding Model Reprogramming for CLIP via Decoupling Visual Prompts},
    author={Chengyi Cai and Zesheng Ye and Lei Feng and Jianzhong Qi and Feng Liu},
    booktitle = {International Conference on Machine Learning},
    year={2025}
}

@InProceedings{pmlr-v139-jia21b,
  title = 	 {Scaling Up Visual and Vision-Language Representation Learning With Noisy Text Supervision},
  author =       {Jia, Chao and Yang, Yinfei and Xia, Ye and Chen, Yi-Ting and Parekh, Zarana and Pham, Hieu and Le, Quoc and Sung, Yun-Hsuan and Li, Zhen and Duerig, Tom},
  booktitle = 	 {Proceedings of the 38th International Conference on Machine Learning},
  pages = 	 {4904--4916},
  year = 	 {2021},
  editor = 	 {Meila, Marina and Zhang, Tong},
  volume = 	 {139},
  series = 	 {Proceedings of Machine Learning Research},
  month = 	 {18--24 Jul},
  publisher =    {PMLR},
  url = 	 {https://proceedings.mlr.press/v139/jia21b.html}
}

@article{clip-adapter,
author = {Gao, Peng and Geng, Shijie and Zhang, Renrui and Ma, Teli and Fang, Rongyao and Zhang, Yongfeng and Li, Hongsheng and Qiao, Yu},
title = {CLIP-Adapter: Better Vision-Language Models with Feature Adapters},
year = {2024},
issue_date = {Feb 2024},
publisher = {Kluwer Academic Publishers},
address = {USA},
volume = {132},
number = {2},
issn = {0920-5691},
url = {https://doi.org/10.1007/s11263-023-01891-x},
doi = {10.1007/s11263-023-01891-x},
journal = {International Journal of Computer Vision},
pages = {581–595},
numpages = {15}
}

@InProceedings{logitdeconfusion_2025_CVPR,
    author    = {Li, Shuo and Liu, Fang and Hao, Zehua and Wang, Xinyi and Li, Lingling and Liu, Xu and Chen, Puhua and Ma, Wenping},
    title     = {Logits DeConfusion with CLIP for Few-Shot Learning},
    booktitle = {Proceedings of the Computer Vision and Pattern Recognition Conference (CVPR)},
    month     = {June},
    year      = {2025},
    pages     = {25411-25421}
}

@inproceedings{tip_adapter,
author = {Zhang, Renrui and Zhang, Wei and Fang, Rongyao and Gao, Peng and Li, Kunchang and Dai, Jifeng and Qiao, Yu and Li, Hongsheng},
title = {Tip-Adapter: Training-Free Adaption of CLIP for Few-Shot Classification},
year = {2022},
isbn = {978-3-031-19832-8},
publisher = {Springer-Verlag},
address = {Berlin, Heidelberg},
url = {https://doi.org/10.1007/978-3-031-19833-5_29},
doi = {10.1007/978-3-031-19833-5_29},
booktitle = {Computer Vision – ECCV 2022: 17th European Conference, Tel Aviv, Israel, October 23–27, 2022, Proceedings, Part XXXV},
pages = {493–510},
numpages = {18},
location = {Tel Aviv, Israel}
}

@INPROCEEDINGS{proto_clip,
  author={P, Jishnu Jaykumar and Palanisamy, Kamalesh and Chao, Yu-Wei and Du, Xinya and Xiang, Yu},
  booktitle={2024 IEEE/RSJ International Conference on Intelligent Robots and Systems (IROS)}, 
  title={Proto-CLIP: Vision-Language Prototypical Network for Few-Shot Learning}, 
  year={2024},
  volume={},
  number={},
  pages={2594-2601},
  doi={10.1109/IROS58592.2024.10801660}}

@INPROCEEDINGS{resnet,
  author={He, Kaiming and Zhang, Xiangyu and Ren, Shaoqing and Sun, Jian},
  booktitle={2016 IEEE Conference on Computer Vision and Pattern Recognition (CVPR)}, 
  title={Deep Residual Learning for Image Recognition}, 
  year={2016},
  volume={},
  number={},
  pages={770-778},
  doi={10.1109/CVPR.2016.90}}

@article{dosovitskiy2020vit,
  title={An Image is Worth 16x16 Words: Transformers for Image Recognition at Scale},
  author={Dosovitskiy, Alexey and Beyer, Lucas and Kolesnikov, Alexander and Weissenborn, Dirk and Zhai, Xiaohua and Unterthiner, Thomas and  Dehghani, Mostafa and Minderer, Matthias and Heigold, Georg and Gelly, Sylvain and Uszkoreit, Jakob and Houlsby, Neil},
  journal={ICLR},
  year={2021}
}

@INPROCEEDINGS{repvgg,
  author={Ding, Xiaohan and Zhang, Xiangyu and Ma, Ningning and Han, Jungong and Ding, Guiguang and Sun, Jian},
  booktitle={2021 IEEE/CVF Conference on Computer Vision and Pattern Recognition (CVPR)}, 
  title={RepVGG: Making VGG-style ConvNets Great Again}, 
  year={2021},
  volume={},
  number={},
  pages={13728-13737},
  doi={10.1109/CVPR46437.2021.01352}}

@misc{luo2023towards,
      title={Towards Efficient Visual Adaption via Structural Re-parameterization}, 
      author={Gen Luo and Minglang Huang and Yiyi Zhou and Xiaoshuai Sun and Guannan Jiang and Zhiyu Wang and Rongrong Ji},
      year={2023},
      eprint={2302.08106},
      archivePrefix={arXiv},
      primaryClass={cs.CV},
      url={https://arxiv.org/abs/2302.08106}, 
}

@inproceedings{yu2023task,
  title={Task Residual for Tuning Vision-Language Models},
  author={Yu, Tao and Lu, Zhihe and Jin, Xin and Chen, Zhibo and Wang, Xinchao},
  booktitle={Proceedings of the IEEE/CVF Conference on Computer Vision and Pattern Recognition},
  pages={10899--10909},
  year={2023}
}

@inproceedings{huang2024lp,
    title={LP++: A Surprisingly Strong Linear Probe for Few-Shot CLIP},
    author={Yunshi Huang and Fereshteh Shakeri and Jose Dolz and Malik Boudiaf and Houda Bahig and Ismail Ben Ayed},
    booktitle={IEEE/CVF Conference on Computer Vision and Pattern Recognition (CVPR)},
    year={2024}
    }

@InProceedings{wu2026dga,
    author    = {Wu, Jiayang and Chen, Xinyang and Lv, Ke and Guan, Weili},
    title     = {Boosting Visual Reprogramming for CLIP with Dual Granularity Alignment},
    booktitle = {Proceedings of the IEEE/CVF Conference on Computer Vision and Pattern Recognition (CVPR)},
    month     = {June},
    year      = {2026},
    pages     = {29347-29356}
}
\bibliographystyle{iclr2027_conference}

\appendix
\newpage

\section{Dataset Information}
\label{app:data_info}

\begin{table*}[h]
\centering
\caption{Summary of the 11 downstream benchmark datasets used in our experiments, including task type, number of classes, and training batch size.}
\vspace{0.15cm}
\resizebox{\textwidth}{!}{
\begin{tabular}{c|ccccccccccc}
\toprule

 & \sc Aircraft      & \sc Caltech       &\sc Cars          &\sc DTD      &\sc ESAT       &\sc Flowers       &\sc Food          &\sc Pets          &\sc SUN           &\sc UCF            &\sc Resisc              \\

 \midrule
 \makecell{\sc Task \\ \sc Info.} & \makecell{aircraft \\ model} & object & \makecell{fine-grained \\ automobile} & texture & \makecell{remote sensing \\ land cover} & flower & food & pet & scene & action &  \makecell{remote \\ sensing scene} \\
 \makecell{\sc Class Number } & 100      & 100       & 196          & 47      & 10       & 102       & 101          & 37          & 397           & 101            & 45              \\
  \makecell{\sc Batch Size } & 64      & 64       & 64          & 64      & 64       & 64       & 64          & 64          & 64           & 64           & 64              \\
\bottomrule
\end{tabular}}
\label{tab:data_info}
\end{table*}

Following prior works~\citep{cai2025attribute,cai2025understanding}, we adopt the same 16-shot benchmark protocol. The benchmark covers 11 publicly available datasets spanning diverse recognition tasks, including fine-grained object recognition, generic object classification, texture recognition, scene understanding, action recognition, and remote sensing. Specifically, we use FGVC Aircraft (Aircraft)~\citep{aircraft}, Caltech101 (Caltech)~\citep{caltech101}, StanfordCars (Cars)~\citep{stanfordcars}, Describable Textures Dataset (DTD)~\citep{dtd}, EuroSAT (ESAT)~\citep{eurosat}, Flowers102 (Flowers)~\citep{flowers}, Food101 (Food)~\citep{food101}, OxfordPets (Pets)~\citep{parkhi2012cats}, SUN397 (SUN)~\citep{sun397}, UCF101 (UCF)~\citep{ucf101}, and RESISC45 (Resisc)~\citep{resisc}. As summarized in \cref{tab:data_info}, these datasets vary substantially in semantic granularity and class cardinality, ranging from 10 classes in EuroSAT to 397 classes in SUN397. Unless otherwise specified, we use a batch size of 64 for training visual reprogramming on all datasets.

\section{Additional Results}
\subsection{Additional Results on Different Backbones}

\begin{table*}[h]
\centering
\caption{Accuracy comparison of different methods trained on 16-shot downstream classification tasks, using RN101-based CLIP as the pretrained model (Mean \%, ours are \textcolor{black}{\colorbox{gray!30}{highlighted}} and the highest is in \textbf{bold}).}
\begin{sc}
\resizebox{\textwidth}{!}{
\begin{tabular}{c|ccccccccccc|c}
\toprule
Method   & Aircraft      & Caltech       & Cars          & DTD           & ESAT          & Flowers       & Food          & Pets          & SUN           & UCF     & Resisc        & Avg.                           \\
\midrule
 VP   & 19.3          & 83.0          & 53.7          & 43.4          & 62.8          & 57.2          & 71.2          & 80.2          & 53.5          & 54.2          & 54.0          & 57.5 \\
 AR     & 19.5          & 89.7          & 62.0          & 46.3          & 70.4          & 60.4          & 78.0          & 84.4          & 58.4          & 60.6          & 60.2          & 62.7 \\
AttrVR & 23.3          & 92.0          & 62.2          & 55.6          & 70.3          & 76.2          & \textbf{79.5} & 89.3          & 62.1          & 64.5          & 64.5          & 67.2 \\
DVP & 23.8          & 92.7          & 62.5          & 58.0          & \textbf{70.7}          & 80.6          & 79.1          & 89.5          & {63.7} & {68.1} & {68.4} & 68.8 \\
\rowcolor{gray!30} 
RVP & \textbf{30.4} & \textbf{93.9} & \textbf{76.0} & \textbf{60.1} & 67.6 & \textbf{90.5} & 76.8 & \textbf{91.1} & \textbf{67.5} & \textbf{75.8} & \textbf{73.4} & \textbf{73.0} \\
\bottomrule
\end{tabular}}
\end{sc}
\label{tab:rn101res}
\end{table*}

\begin{table*}[!h]
\centering
\caption{Accuracy comparison of different methods trained on 16-shot downstream classification tasks, using ViT-B/32-based CLIP as the pretrained model (Mean \%, ours are \textcolor{black}{\colorbox{gray!30}{highlighted}} and the highest is in \textbf{bold}).}
\begin{sc}
\resizebox{\textwidth}{!}{
\begin{tabular}{c|ccccccccccc|c}
\toprule
Method                & Aircraft      & Caltech       & Cars          & DTD           & ESAT          & Flowers       & Food          & Pets          & SUN           & UCF     & Resisc        & Avg.                           \\
\midrule
 VP  & 24.3          & 92.3          & 58.6          & 54.9          & 85.9          & 71.2          & 75.0          & 86.8          & 61.0          & 67.3          & 73.9          & 68.3 \\
 AR    & 21.8          & 92.7          & 56.9          & 49.9          & 85.6          & 66.7          & 75.7          & 84.7          & 59.9          & 63.5          & 71.6          & 66.3 \\
AttrVR & 24.5          & 92.0          & 56.6          & 56.8          & \textbf{88.6}          & 77.8          & \textbf{77.2} & 89.8          & 62.8          & 67.9          & 73.9          & 69.8 \\
DVP & 26.1          & 92.9          & 56.5          & 57.2          & 88.5          & 82.5          & 77.0          & 89.2          & {64.2} & {70.5} & {76.0} & 71.0 \\
\rowcolor{gray!30}
RVP & \textbf{32.8} & \textbf{94.1} & \textbf{74.1} & \textbf{63.4} & 86.4 & \textbf{93.3} & 74.8 & \textbf{90.3} & \textbf{68.0} & \textbf{77.8} & \textbf{79.1} & \textbf{75.8} \\
\bottomrule
\end{tabular}}
\end{sc}
\label{tab:vitb32res}
\end{table*}

We further evaluate RVP on RN101- and ViT-B/32-based CLIP backbones in \cref{tab:rn101res,tab:vitb32res}. The results remain consistent with the main experiments: RVP achieves the best average accuracy on both backbones and outperforms all previous visual reprogramming baselines by a clear margin.

With the RN101 backbone, RVP reaches an average accuracy of {73.0\%}, improving over DVP by \(4.2\) points and over AttrVR by \(5.8\) points. It achieves the best result on 9 out of 11 datasets, with especially large gains on Aircraft (\(+6.6\) over DVP), Cars (\(+13.5\)), Flowers (\(+9.9\)), UCF (\(+7.7\)), and Resisc (\(+5.0\)). These results again show that the proposed structured mapping is particularly effective when the downstream task requires distinguishing semantically similar categories. The only datasets where RVP does not achieve the best performance are ESAT and Food, where the advantage of explicit inter-class modeling appears less pronounced.

A similar pattern is observed for the ViT-B/32 backbone. RVP obtains the highest average accuracy of {75.8\%}, surpassing DVP by \(4.8\) points and AttrVR by \(6.0\) points. It performs best on 9 out of 11 datasets and shows especially large improvements on Aircraft (\(+6.7\) over DVP), Cars (\(+17.6\)), Flowers (\(+10.8\)), UCF (\(+7.3\)), and DTD (\(+6.2\)). Notably, the gain on Cars remains very large even with the stronger ViT-based encoder, further supporting our claim that explicit modeling of inter-class relationships is particularly beneficial for fine-grained recognition.

Taken together, these additional results strengthen two observations from the main paper. First, the advantage of RVP is robust across both convolutional and transformer backbones. Second, the largest improvements consistently appear on fine-grained datasets such as Aircraft and Cars, where many categories share highly similar semantic attributes and cannot be reliably separated by intra-class prompt selection alone. This further supports the central motivation of RVP: when the pretrained feature space contains strong inter-class correlation, a structured mapping that explicitly models class relationships provides a more effective adaptation mechanism than independent prompt aggregation.

\subsection{Broader Comparison}

Although RVP follows a different adaptation paradigm from conventional CLIP adaptation methods, we further compare it with several representative approaches under the same 16-shot setting. These methods include prompt learning, feature adaptation, task residual learning, and linear probing, while RVP performs adaptation through visual reprogramming with structured label mapping.

As shown in Table~\ref{tab:more}, RVP achieves an average accuracy of $82.7\%$, showing competitive performance across the 11 datasets. It performs particularly well on Aircraft and Cars, reaching $46.1\%$ and $84.8\%$, respectively, while also obtaining strong results on Caltech, EuroSAT, Pets, UCF, and RESISC. Although the compared methods use different adaptation strategies, this broader comparison shows that RVP remains effective when evaluated alongside general CLIP adaptation approaches.

\begin{table*}[h]
    \centering
    \caption{Accuracy comparison of different methods trained on 16-shot downstream classification tasks, using ViT-B/16-based CLIP as the pretrained model (Mean \% $\pm$ Std \%, ours are \textcolor{black}{\colorbox{gray!30}{highlighted}} and the highest result is in \textbf{bold}).}
    \begin{sc}
    \resizebox{\textwidth}{!}{
    \begin{tabular}{c|ccccccccccc|c}
    \toprule
    Method
    & Aircraft
    & Caltech
    & Cars
    & DTD
    & EuroSAT
    & Flowers
    & Food
    & Pets
    & SUN
    & UCF
    & RESISC
    & Avg. \\
    \midrule

    CoOp
    & 43.2
    & 95.8
    & 82.9
    & 69.7
    & 85.0
    & 96.8
    & 84.2
    & 92.0
    & 74.9
    & 83.1
    & 84.7
    & 81.1 \\

    CoCoOp
    & 33.3
    & 95.1
    & 72.3
    & 63.7
    & 73.6
    & 89.1
    & \textbf{87.4}
    & 93.4
    & 72.6
    & 77.2
    & 81.6
    & 76.3 \\

    CLIP-Adapter
    & 34.2
    & 94.9
    & 74.0
    & 59.4
    & 71.4
    & 92.9
    & 87.1
    & 92.3
    & 74.2
    & 80.2
    & 85.7
    & 76.9 \\

    Tip-Adapter-F
    & 44.6
    & 95.7
    & 82.3
    & 70.8
    & 85.9
    & 96.2
    & 86.8
    & 92.6
    & 76.0
    & 83.9
    & 81.2
    & 81.5 \\

    TaskRes
    & 44.9
    & 95.8
    & 83.5
    & 71.5
    & 82.7
    & \textbf{97.5}
    & 86.9
    & 92.4
    & \textbf{76.1}
    & 84.0
    & 83.3
    & 81.7 \\

    LP++
    & 42.1
    & 95.8
    & 80.8
    & \textbf{71.9}
    & 85.5
    & 96.3
    & 87.2
    & 92.6
    & 76.0
    & 83.9
    & 80.9
    & 81.2 \\

    \rowcolor{gray!30}
    RVP
    & \textbf{46.1}\scriptsize{$\pm$0.2}
    & \textbf{96.5}\scriptsize{$\pm$0.2}
    & \textbf{84.8}\scriptsize{$\pm$0.3}
    & 68.7\scriptsize{$\pm$0.1}
    & \textbf{92.7}\scriptsize{$\pm$0.2}
    & 96.7\scriptsize{$\pm$0.2}
    & 85.5\scriptsize{$\pm$0.1}
    & \textbf{94.0}\scriptsize{$\pm$0.1}
    & 73.8\scriptsize{$\pm$0.1}
    & \textbf{85.1}\scriptsize{$\pm$0.7}
    & \textbf{85.9}\scriptsize{$\pm$0.5}
    & \textbf{82.7} \\

    \bottomrule
    \end{tabular}}
    \end{sc}
    \label{tab:more}
    \vspace{-0.1cm}
\end{table*}

\begin{table}[t]
    \centering
    \caption{\textbf{Accuracy and parameter efficiency on StanfordCars} under the 16-shot setting with ViT-B/16. Trainable parameters count only method-specific adaptation parameters.}
    \label{tab:efficiency_cars}
    \small
    \resizebox{\columnwidth}{!}{
    \begin{tabular}{l|cc|l}
        \toprule
        Method
        &  Accuracy (\%)
        & Trainable Params.
        & Inference Path \\
        \midrule

        CoOp
        & 82.9
        & 0.008M
        & Fixed classifier from learned text prompts \\

        CoCoOp
        & 72.3
        & 0.042M
        & Image-conditioned text features \\

        CLIP-Adapter
        & 74.0
        & 0.131M
        & Nonlinear feature adapter \\

        Tip-Adapter-F
        & 82.3
        & 1.606M
        & Cache-based adapted logits \\

        TaskRes
        & 83.5
        & 0.100M
        & Fixed linear head \\

        LP++
        & 80.8
        & 0.101M
        & Fixed linear head \\

        LDC
        & 84.2
        & 5.336M
        & Multi-level adapters and adaptive fusion \\

        \rowcolor{gray!30}
        RVP
        & \textbf{84.8}
        & 0.082M
        & Prompted CLIP with exactly folded linear head \\

        \bottomrule
    \end{tabular}}
\end{table}

As shown in Table~\ref{tab:efficiency_cars}, RVP achieves a favorable accuracy--parameter trade-off. In particular, it slightly improves over LDC on StanfordCars while using only about $1.5\%$ of its trainable parameters. This efficiency follows from the structured design of RVP, whose learned mapping can be exactly reparameterized into a single linear head at inference. The results therefore show that RVP can retain strong recognition performance without relying on a large adaptation module.

\subsection{Computation Cost}
The VP method \citep{bahng2022exploring} adopts a visual noise pattern with a frame width of 30 pixels. For an input image of size $224\times224$, this corresponds to
$224\times224\times3 - (224-60)\times(224-60)\times3 = 69840$
trainable prompt parameters. In contrast, both AR \citep{tsai2020transfer,chen2023understanding} and AttrVR \citep{cai2025attribute} use a narrower frame width of 16 pixels, which results in
$224\times224\times3 - (224-32)\times(224-32)\times3 = 39936$
trainable parameters. DVP employs decoupled visual prompting, typically using three visual prompts. Its total number of prompt parameters is therefore $39936 \times 3 = 119808.$ Just like AR \citep{tsai2020transfer,chen2023understanding} and AttrVR \citep{cai2025attribute}, RVP uses 39936 trainable prompt parameters.

Regarding logit aggregation, VP and AR do not introduce additional trainable parameters, as they do not rely on multiple textual descriptions per class. AttrVR adopts fixed aggregation functions (e.g., mean, average, max, or kNN), which also do not introduce learnable parameters. In contrast, DVP employs a Probability Reweighting Matrix
$\boldsymbol{\omega}_{\mathrm{PRM}}\in\mathbb{R}^{CM\times C}$
for aggregating description-level logits. Although the matrix is
defined over $CM\times C$ entries, only $C\times M$ mapping parameters
are effectively learnable under its structured parameterization.

For RVP, the intra-class aggregation matrix $P \in \mathbb{R}^{C\times M}$ introduces $C\times M$ parameters, while the inter-class matrix $E \in \mathbb{R}^{C\times C}$ contributes an additional $C\times C$ parameters.

\subsection{Additional Analysis}
\paragraph{Inter-class structure and the benefit of $\mathbf{E}$.}
To examine when inter-class correction is most beneficial, we characterize the
class structure induced by the text embeddings. For each dataset, we uniformly
aggregate the 20 normalized attribute embeddings of each class, compute the
eigenspectrum of the resulting class-prototype Gram matrix, and define
$r_{90}$ as the minimum number of eigenvalues required to explain $90\%$ of
the spectral mass. Since the number of classes $C$ varies substantially across
datasets, from 10 to 397, we use the normalized quantity $r_{90}/C$ as the
primary statistic. We measure the benefit of inter-class modeling as
$\Delta_E=\mathrm{Acc}(\mathrm{RVP})-\mathrm{Acc}(\mathrm{w/o}\ E)$.

\begin{table}[t]
    \centering
    \caption{\textbf{Relationship between text-space concentration and the
    benefit of inter-class correction.}
    $r_{90}$ denotes the minimum number of eigenvalues explaining $90\%$ of
    the spectral mass of the class-prototype Gram matrix, and
    $\Delta_E$ measures the accuracy gain from enabling the inter-class
    residual matrix $\mathbf{E}$.}
    \label{tab:spectrum_e}
    \small
    \resizebox{0.5\columnwidth}{!}{
    \begin{tabular}{l|rrr|rrr}
        \toprule
        Dataset
        & $C$
        & $r_{90}$
        & $r_{90}/C$
        & RVP
        & w/o $\mathbf{E}$
        & $\Delta_E$ \\
        \midrule
        Aircraft    & 100 & 7  & 0.070 & 46.1 & 35.6 & 10.5 \\
        Caltech101  & 100 & 28 & 0.280 & 96.5 & 96.1 & 0.4 \\
        Cars        & 196 & 24 & 0.122 & 84.8 & 68.0 & 16.8 \\
        DTD         & 47  & 3  & 0.064 & 68.7 & 63.3 & 5.4 \\
        EuroSAT     & 10  & 1  & 0.100 & 92.7 & 93.8 & -1.1 \\
        Flowers102  & 102 & 27 & 0.265 & 96.7 & 91.7 & 5.0 \\
        Food101     & 101 & 29 & 0.287 & 85.5 & 85.6 & -0.1 \\
        Oxford Pets & 37  & 11 & 0.297 & 94.0 & 93.1 & 0.9 \\
        SUN397      & 397 & 31 & 0.078 & 73.8 & 67.4 & 6.4 \\
        UCF101      & 101 & 22 & 0.218 & 85.1 & 78.9 & 6.2 \\
        RESISC45    & 45  & 7  & 0.156 & 85.9 & 83.5 & 2.4 \\
        \bottomrule
    \end{tabular}}
\end{table}

Across the 11 datasets, $r_{90}/C$ is negatively associated with
$\Delta_E$ (Spearman $\rho=-0.527$). Since $r_{90}/C$ can also depend on the
number of classes, we additionally compute a partial Spearman correlation
while controlling for $C$, which yields $\rho=-0.625$ with $p=0.040$.
The relationship is also stable under leave-one-dataset-out analysis: all
partial correlations remain negative, ranging from $-0.772$ to $-0.510$.
These observations are consistent with the hypothesis that when class
prototypes occupy a more concentrated shared text subspace, there is more
room for inter-class correction to improve class discrimination.

\paragraph{Large gains and the learned structure of $\mathbf{E}$.}
The large improvement on StanfordCars is not explained by $\mathbf{E}$ acting
as a negligible residual. On this dataset, $\mathbf{E}$ contains 38,416
parameters, while the 16-shot training set contains 3,136 images. On held-out
data, the mean relative correction induced by $\mathbf{E}$ is $0.456$, and
enabling $\mathbf{E}$ changes $36.2\%$ of predictions. At the same time, the
learned matrix exhibits clear structure: its stable rank is only $15.6$,
compared with $50.1\pm1.0$ under an entry-permutation null. Moreover,
$\|\operatorname{diag}(\mathbf{E})\|_F/\|\mathbf{E}\|_F=0.075$, close to the
null value of $0.072$, indicating that the learned correction is predominantly
off-diagonal and therefore genuinely inter-class rather than a simple
per-class rescaling.

\begin{table}[h]
\centering
\caption{Number of trainable parameters for different methods on the Aircraft dataset ($C=100$, $M=20$).}
\label{tab:param_comparison}
\footnotesize
\begin{tabular}{l|c|c|c}
\toprule
\textbf{Method} & \textbf{Prompt Parameters}  & \textbf{Total} & Accuracy \\
\midrule
VP & 69840 & 69840 & 32.1 \\
AR & 39936 & 39936 & 31.7 \\
AttrVR & 39936 & 39936 & 36.6 \\
DVP & 119808 & 121808 & 38.7 \\
RVP & 39936 & 51936 & 46.1 \\
\bottomrule
\end{tabular}%
\end{table}

As shown in \cref{tab:param_comparison}, RVP achieves the best accuracy on Aircraft while remaining parameter-efficient. Although it uses the same number of prompt parameters as AR and AttrVR, its additional structured aggregation introduces only a modest overhead, resulting in a total of 51936 parameters. In contrast, DVP uses substantially more parameters due to multiple visual prompts, yet still underperforms RVP. This shows that the gain of RVP comes from a more effective parameterization of class relationships rather than from simply increasing model size.

\begin{table*}[h]
\centering
\caption{Accuracy comparison of RVP, AttrVR, and DVP under the same text prompt setting, where DVP is restricted to a single group of trainable visual prompts, using ViT-B/16 CLIP as the pretrained model (mean \%; ours are \textcolor{black}{\colorbox{gray!30}{highlighted}} and the best results are shown in \textbf{bold}).}
\begin{sc}
\resizebox{\textwidth}{!}{
\begin{tabular}{c|ccccccccccc|c}
\toprule
                 & Aircraft      & Caltech       & Cars          & DTD           & ESAT          & Flowers       & Food          & Pets          & SUN           & UCF           & Resisc        & Avg. \\
\midrule
AttrVR (DesAttr) & 35.9          & 95.6          & 68.2          & 64.4          & 93.8          & 92.4          & \textbf{85.7}          & 93.0          & 67.7          & 78.6          & 81.8          & 77.9 \\
DVP (num=1)  & 36.4          & 95.8          & 69.1          & 65.3          & \textbf{94.1} & 93.6          & \textbf{85.7} & 93.1          & {70.0} & {80.2} & {82.8} & {78.7} \\
\rowcolor{gray!30}
RVP & \textbf{46.1} & \textbf{96.5} & \textbf{84.8} & \textbf{68.7} & 92.7 & \textbf{96.7} & 85.5 & \textbf{94.0} & \textbf{73.8} & \textbf{85.1} & \textbf{85.9} & \textbf{82.7} \\
\bottomrule
\end{tabular}}
\end{sc}
\label{tab:1VP}
\end{table*}

\cref{tab:1VP} compares RVP, AttrVR, and DVP under a controlled setting where all methods use the same text prompts and DVP is restricted to a single group of trainable visual prompts. Specifically, since AttrVR originally uses two groups of text prompts, namely \textit{Descriptive Attributes} and \textit{Distinctive Attributes}, we retain only \textit{Descriptive Attributes} here to ensure a fair comparison across methods. Under this setting, RVP still achieves clear and consistent improvements over both AttrVR and DVP. In particular, RVP attains the best average accuracy of 82.7\%, outperforming DVP by 4.0 points and AttrVR by 4.8 points. The improvement is especially pronounced on fine-grained datasets such as Aircraft and Cars, where RVP surpasses DVP by 9.7 and 15.7 points, respectively. These results indicate that the advantage of RVP does not rely on using more diverse text prompts or multiple prompt groups. Instead, the gain comes from its more effective modeling of intra-class aggregation and inter-class relationships, which allows it to better suppress shared semantics and enhance subtle class-specific differences.

\begin{table*}[h]
\centering
\caption{Accuracy comparison of our RVP and DVPlite trained on 16-shot downstream classification task, using ViT-B/16-based CLIP as the pretrained model (Mean \%, ours is \textcolor{black}{\colorbox{gray!30}{highlighted}} and the highest is in \textbf{bold}).}
\footnotesize
\begin{sc}
\resizebox{\textwidth}{!}{
\begin{tabular}{c|ccccccccccc|c}
\toprule
                & Aircraft & Caltech & Cars & DTD  & ESAT & Flowers & Food & Pets & SUN  & UCF & Resisc & Avg. \\
\midrule
AttrVR & 36.6     & 95.7    & 68.3 & 65.6 & \textbf{93.8} & 92.9    & \textbf{85.9} & 93.3 & 69.6 & 79.0  & 82.6   & 78.5    \\
DVPlite         & {39.3}     & {95.9}    & {71.4} & {66.5} & \textbf{93.8} & {95.2}    & 85.8 & {93.4} & {71.6} & {81.0}  & {83.6}   & {79.8} \\
\rowcolor{gray!30}
RVP & \textbf{46.1} & \textbf{96.5} & \textbf{84.8} & \textbf{68.7} & 92.7 & \textbf{96.7} & 85.5 & \textbf{94.0} & \textbf{73.8} & \textbf{85.1} & \textbf{85.9} & \textbf{82.7} \\
\bottomrule
\end{tabular}}
\end{sc}
\label{tab:dvplite}
\end{table*}

\cref{tab:dvplite} compares our method with DVPlite, an efficient variant of DVP proposed by Cai et al.~\citep{cai2025understanding}. DVPlite decomposes the visual prompt into four directional components, namely up, down, left, and right, and assigns them to different cause groups generated by an LLM. Unlike standard DVP, this design avoids multiple forward passes through the image encoder. However, it still relies on substantially more text prompts than our method, since each direction is associated with its own set of $C\times M$ prompts. 
Despite this more complex prompt design, our method achieves the best overall performance. As shown in \cref{tab:dvplite}, RVP attains an average accuracy of 82.7\%, outperforming DVPlite by 2.9 points. These results show that RVP is not only more accurate, but also simpler to use, as it avoids directional prompt decomposition and additional LLM-based cause grouping while still delivering stronger performance.

\subsection{Additional Visualization}
\begin{figure*}[t]
    \centering
    \includegraphics[width=\linewidth]{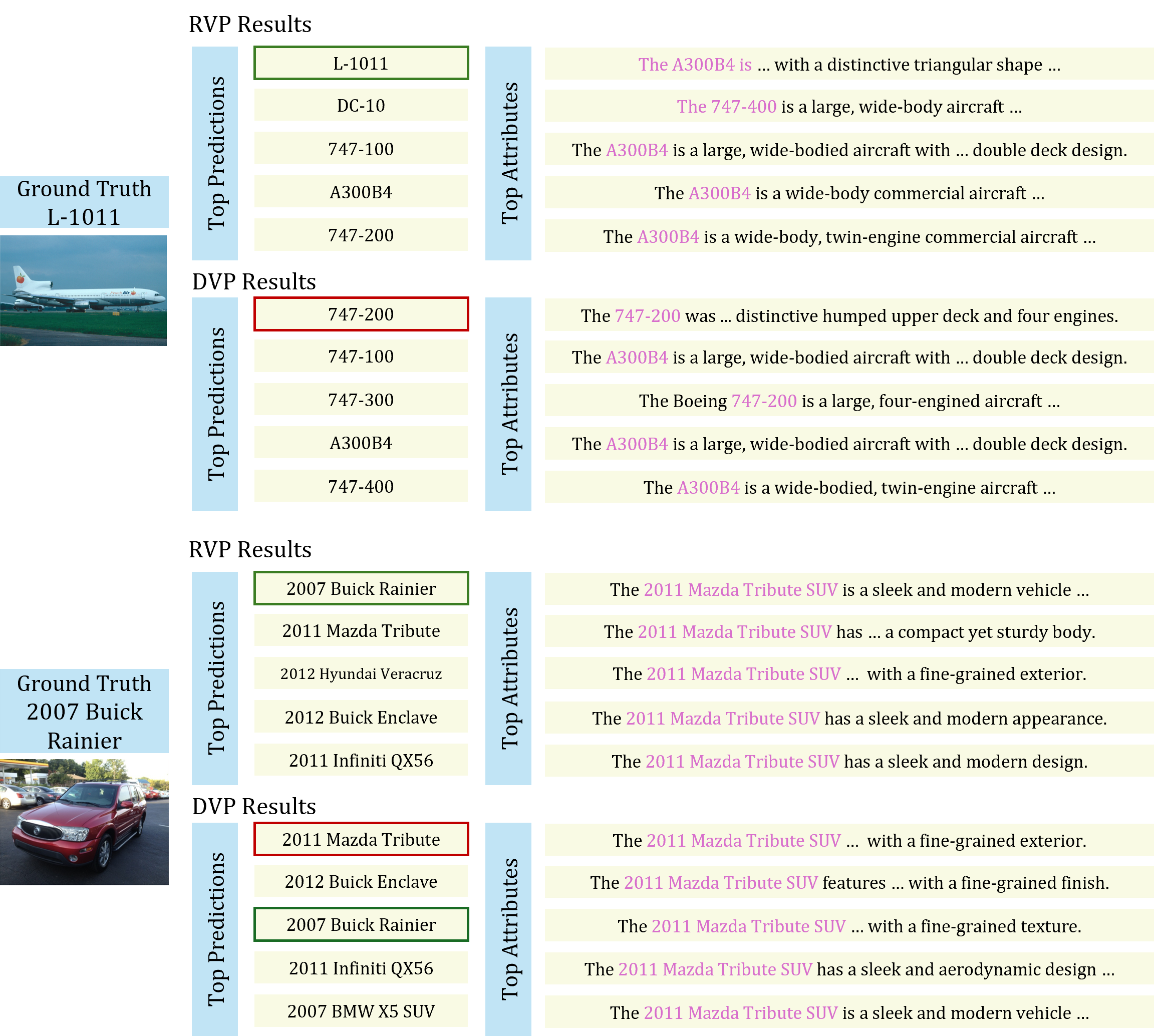}
    \caption{Additional visualization of the top predicted classes and highest-matching attributes for a test image. For both RVP and DVP, the most similar attributes include prompts from other classes, reflecting strong semantic overlap in fine-grained recognition. However, RVP explicitly models inter-class relationships, allowing it to better resolve these cross-class ambiguities and produce the correct prediction.}
    \label{fig:vis_air_car}
\end{figure*}

\cref{fig:vis_air_car} further shows that attribute matching alone is not sufficient for fine-grained recognition. In both the aircraft and car examples, the top-matched attributes are dominated by semantically similar but incorrect classes, indicating that these attributes largely overlap and cannot reliably determine the final label on their own. Despite receiving similarly misleading attribute evidence, RVP still predicts the correct class, whereas DVP fails. This is because RVP does not rely only on prompt-level similarity; instead, it explicitly captures inter-class relationships, allowing it to suppress confusing evidence from correlated classes and produce better predictions.

\subsection{Additional Error Analysis}
\label{app:limitation}

\begin{figure*}[htbp]
    \centering
    \includegraphics[width=\linewidth]{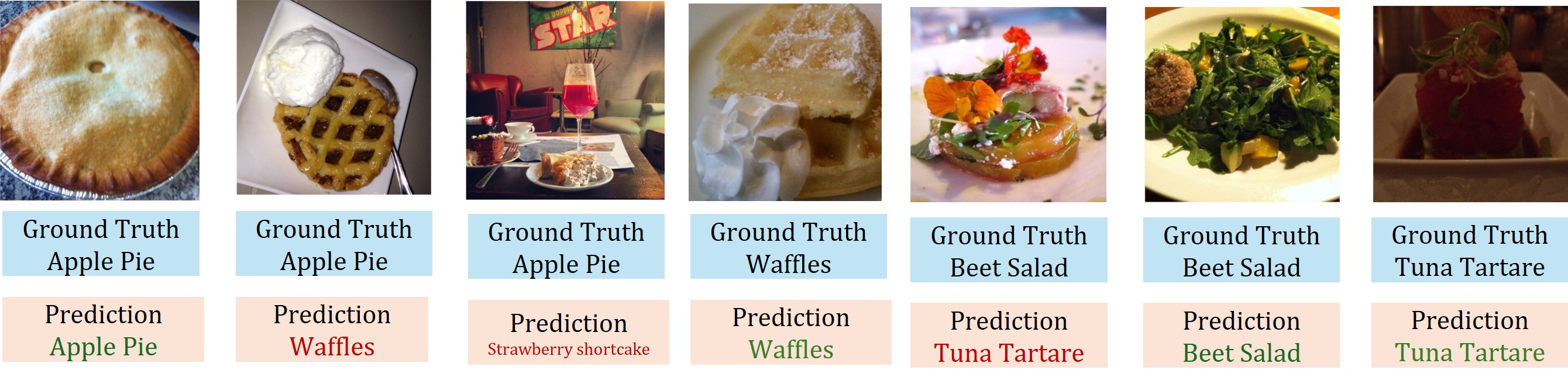}
    \caption{Qualitative examples on Food101. Food categories often exhibit large intra-class variation and strong cross-class visual overlap due to differences in plating, viewpoint, garnish, and accompanying side dishes. As shown here, classes such as Apple Pie and Waffles can be confused when the main dish is partially visible or co-occurs with similar desserts, while Beet Salad and Tuna Tartare may share similar fine-grained presentation and ingredients. Correct predictions are shown in green and incorrect predictions in red.}
    \label{fig:vis_food}
\end{figure*}

Food101 is particularly challenging because its class semantics are often compositional rather than visually stable. Unlike fine-grained object categories, a food image may contain the main dish together with side dishes, garnish, sauces, or additional ingredients, so the same class can vary substantially across samples. As illustrated in \cref{fig:vis_food}, Apple Pie can co-occur with cream or be presented in ways that resemble other desserts, while Beet Salad and Tuna Tartare may share similar plating style, color, and ingredient structure. In such cases, the ambiguity is driven not only by inter-class similarity, but also by high intra-class variation and unstable visual cues, which reduces the benefit of structured inter-class modeling.

\begin{figure*}[htbp]
    \centering
    \includegraphics[width=\linewidth]{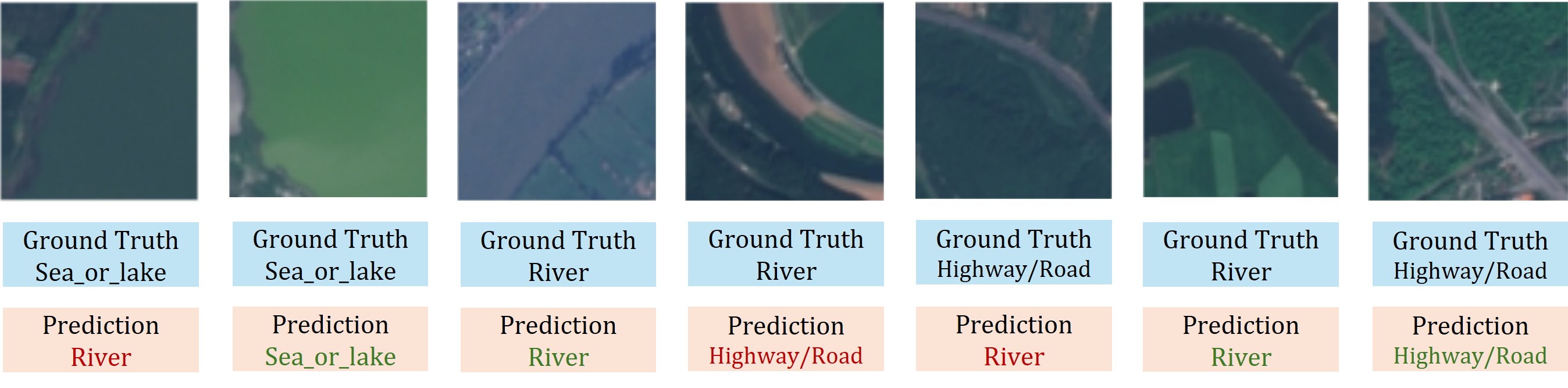}
    \caption{Qualitative examples on EuroSAT. Several classes, especially Sea\_or\_lake, River, and Highway/Road, exhibit strong visual similarity in satellite crops due to elongated shapes, curved boundaries, and limited scene context. As a result, some samples remain ambiguous even under the proposed method, suggesting that EuroSAT is less dominated by inter-class semantic ambiguity than fine-grained recognition benchmarks. Correct predictions are shown in green and incorrect predictions in red.}
    \label{fig:vis_eurosat}
\end{figure*}
A possible reason why RVP is less advantageous on EuroSAT is that this benchmark does not primarily require the kind of inter-class semantic disambiguation that RVP is designed to address. Unlike fine-grained tasks such as Aircraft and Cars, where many classes share highly similar semantic attributes, EuroSAT contains only 10 classes, and many errors arise from coarse visual ambiguity in satellite crops rather than from strong overlap in text semantics. As illustrated in \cref{fig:vis_eurosat}, categories such as Sea\_or\_lake, River, and Highway/Road can appear visually similar due to limited resolution (64 $\times$ 64 pixels), elongated structures, and missing global scene context. In such cases, the main difficulty lies in ambiguous visual evidence and spatial layout, rather than in class relationships within the text embedding space. As a result, explicit inter-class modeling provides less benefit on EuroSAT than on more fine-grained benchmarks.

\section{Propositions and Proof}
\label{sec:proof}
\subsection{Theoretical Justification of RVP}
\label{app:theory_rvp}

In this section, we provide a formal justification for why RVP is well-suited to few-shot visual reprogramming, especially for fine-grained recognition. We do not claim that RVP is universally optimal for all data distributions. Rather, the results below show that under a natural low-rank shared-semantic assumption, RVP has three desirable properties: (i) it preserves the pretrained CLIP semantic subspace, (ii) its inter-class residual can explicitly suppress shared semantic components, and (iii) it refines decision margins in a stable manner.

\paragraph{Setup.}
Recall that the inference rule of RVP is
\begin{equation}
    \mathbf{z} = \hat{\mathbf{v}}^\top \hat{W},
    \qquad
    \hat{W} = \frac{1}{\tau} T^\top W_1 (I + E),
\end{equation}
where \(T \in \mathbb{R}^{CM \times D}\) is the stacked text embedding matrix, \(W_1 \in \mathbb{R}^{CM \times C}\) is the block-diagonal intra-class aggregation matrix, and \(E \in \mathbb{R}^{C \times C}\) is the inter-class residual matrix.

\begin{proposition}[Exact linear reparameterization]
\label{prop:exact_reparam}
For any learned intra-class weights \(W_1\) and inter-class matrix \(E\), there exists a single matrix \(\hat{W} \in \mathbb{R}^{D \times C}\) such that the training-time classifier and the inference-time classifier are identical:
\begin{equation}
    \mathbf{z}
    =
    \frac{1}{\tau}\hat{\mathbf{v}}^\top T^\top W_1 (I+E)
    =
    \hat{\mathbf{v}}^\top \hat{W}.
\end{equation}
\end{proposition}

\begin{proof}
This follows directly from the associativity of matrix multiplication by defining
\[
    \hat{W} = \frac{1}{\tau}T^\top W_1 (I+E).
\]
Substituting this definition into the classifier gives the desired result.
\end{proof}

\begin{proposition}[Text-span preservation]
\label{prop:text_span}
Every column of the reparameterized classifier \(\hat{W}\) lies in the column space of \(T^\top\). Equivalently,
\begin{equation}
    \mathrm{col}(\hat{W}) \subseteq \mathrm{col}(T^\top).
\end{equation}
\end{proposition}

\begin{proof}
By \cref{prop:exact_reparam},
\[
    \hat{W} = \frac{1}{\tau} T^\top W_1 (I+E).
\]
Hence, each column of \(\hat{W}\) is a linear combination of the columns of \(T^\top\).
Therefore,
\[
    \mathrm{col}(\hat{W}) \subseteq \mathrm{col}(T^\top).
\]
\end{proof}

\noindent
\cref{prop:text_span} shows that RVP constructs its classifier entirely from the span of the pretrained CLIP text embeddings, rather than introducing classifier directions outside this text-induced space. Thus, the inter-class correction recombines existing text-derived directions instead of learning an unconstrained classifier directly in \(\mathbb{R}^D\). When the text embeddings occupy a lower-dimensional subspace, this additionally restricts the effective classifier space and provides a structured inductive bias for few-shot adaptation.

\begin{assumption}[Shared-semantic decomposition]
\label{ass:shared_semantic}
Let the base class-logit vector before inter-class correction be \(\mathbf{f}(x) \in \mathbb{R}^C\). Assume that there exists an \(r\)-dimensional subspace \(\mathcal{S} \subset \mathbb{R}^C\), with orthonormal basis \(U \in \mathbb{R}^{C \times r}\), such that
\begin{equation}
    \mathbf{f}(x) = \mathbf{s}(x) + \mathbf{d}(x),
\end{equation}
where \(\mathbf{s}(x) \in \mathcal{S}\) is a shared semantic component and \(\mathbf{d}(x) \in \mathcal{S}^{\perp}\) is a class-discriminative component.
\end{assumption}

\noindent
\cref{ass:shared_semantic} formalizes the empirical observation that fine-grained classes often share dominant semantic directions, while useful class-specific information resides in weaker contrastive components.  

\begin{theorem}[Suppression of shared semantic components]
\label{thm:projection}
Under \cref{ass:shared_semantic}, there exists a residual matrix \(E\) such that the inter-class correction removes the shared semantic component exactly. In particular, if we choose
\begin{equation}
    E = -UU^\top,
\end{equation}
then
\begin{equation}
    \mathbf{z}
    =
    \mathbf{f}(x)(I+E)
    =
    \mathbf{d}(x).
\end{equation}
\end{theorem}

\begin{proof}
Substituting \(E=-UU^\top\) yields
\[
    I+E = I-UU^\top,
\]
which is the orthogonal projector onto \(\mathcal{S}^{\perp}\). Since \(\mathbf{s}(x)\in \mathcal{S}\), we have
\[
    \mathbf{s}(x)(I-UU^\top)=\mathbf{0}.
\]
Since \(\mathbf{d}(x)\in \mathcal{S}^{\perp}\), we have
\[
    \mathbf{d}(x)(I-UU^\top)=\mathbf{d}(x).
\]
Therefore,
\[
    \mathbf{z}
    =
    \big(\mathbf{s}(x)+\mathbf{d}(x)\big)(I-UU^\top)
    =
    \mathbf{d}(x).
\]
\end{proof}

\noindent
\cref{thm:projection} provides a formal explanation for why inter-class modeling is useful. If different classes share a low-rank semantic component, then a suitable residual class-relation matrix can cancel that shared component and retain only the discriminative part.

\begin{corollary}[Margin recovery under the projected classifier]
\label{cor:margin_recovery}
Under the conditions of \cref{thm:projection}, suppose the true label is \(y\) and the discriminative component satisfies
\begin{equation}
    d_y(x) - \max_{j \neq y} d_j(x) > 0.
\end{equation}
Then the RVP classifier with \(E=-UU^\top\) predicts the correct class:
\begin{equation}
    y = \arg\max_{c} z_c.
\end{equation}
\end{corollary}

\begin{proof}
By \cref{thm:projection}, \(\mathbf{z}=\mathbf{d}(x)\). Hence
\[
    z_y - \max_{j\neq y} z_j
    =
    d_y(x) - \max_{j\neq y} d_j(x)
    >0,
\]
which implies \(y=\arg\max_c z_c\).
\end{proof}

\begin{proposition}[Bounded margin degradation under residual correction]
\label{prop:stable_margin}
Let
\begin{equation}
    \mathbf{z} = \mathbf{f} + \Delta,
    \qquad
    \Delta = \mathbf{f}E.
\end{equation}
For any class \(y\), define the multiclass margin
\begin{equation}
    m_y(\mathbf{a}) = a_y - \max_{j\neq y} a_j.
\end{equation}
Then
\begin{equation}
    m_y(\mathbf{z})
    \geq
    m_y(\mathbf{f}) - 2\|\Delta\|_{\infty}.
\end{equation}
\end{proposition}

\begin{proof}
Since \(\mathbf{z}=\mathbf{f}+\Delta\),
\[
    z_y \geq f_y - \|\Delta\|_\infty,
\]
and
\[
    \max_{j\neq y} z_j
    \leq
    \max_{j\neq y} f_j + \|\Delta\|_\infty.
\]
Subtracting the second inequality from the first gives
\[
    m_y(\mathbf{z})
    =
    z_y - \max_{j\neq y} z_j
    \geq
    \big(f_y - \|\Delta\|_\infty\big)
    -
    \big(\max_{j\neq y} f_j + \|\Delta\|_\infty\big),
\]
which simplifies to
\[
    m_y(\mathbf{z})
    \geq
    m_y(\mathbf{f}) - 2\|\Delta\|_\infty.
\]
\end{proof}

\noindent
\cref{prop:stable_margin} bounds the possible degradation of the classification margin under residual correction: the margin can decrease from the base margin by at most \(2\|\mathbf{f}E\|_{\infty}\). In particular, if
\[
    m_y(\mathbf{f}) > 2\|\mathbf{f}E\|_{\infty},
\]
then \(m_y(\mathbf{z})>0\), and the original prediction for class \(y\) is preserved.

\begin{proposition}[Structured restriction of the hypothesis class]
\label{prop:hypothesis_class}
Let
\begin{equation}
    \mathcal{H}_{\mathrm{dense}}
    =
    \left\{
    \hat{\mathbf{v}} \mapsto
    \frac{1}{\tau}\hat{\mathbf{v}}^\top T^\top \Omega
    \;:\;
    \Omega \in \mathbb{R}^{CM \times C}
    \right\}
\end{equation}
be the class of dense description-to-class mappings, and let
\begin{equation}
    \mathcal{H}_{\mathrm{RVP}}
    =
    \left\{
    \hat{\mathbf{v}} \mapsto
    \frac{1}{\tau}\hat{\mathbf{v}}^\top T^\top W_1(I+E)
    \;:\;
    W_1 \text{ is block diagonal},\;
    E \in \mathbb{R}^{C\times C}
    \right\}.
\end{equation}
Then
\begin{equation}
    \mathcal{H}_{\mathrm{RVP}}
    \subseteq
    \mathcal{H}_{\mathrm{dense}}.
\end{equation}
Moreover, RVP contains \(CM+C^2\) trainable parameters in its structured output mapping, whereas the dense mapping contains \(C^2M\) trainable parameters.
\end{proposition}

\begin{proof}
For any \(W_1\) and \(E\), define
\[
    \Omega = W_1(I+E).
\]
Since
\[
    \Omega \in \mathbb{R}^{CM\times C},
\]
every function in \(\mathcal{H}_{\mathrm{RVP}}\) is also an element of
\(\mathcal{H}_{\mathrm{dense}}\). Hence,
\[
    \mathcal{H}_{\mathrm{RVP}}
    \subseteq
    \mathcal{H}_{\mathrm{dense}}.
\]

For the parameter count, \(W_1\) is determined by \(C\) groups of \(M\)
intra-class weights, corresponding to \(CM\) trainable parameters, while
\(E\) contributes \(C^2\) trainable parameters. Therefore, the structured
RVP output mapping contains \(CM+C^2\) trainable parameters. In contrast,
the dense matrix
\(\Omega\in\mathbb{R}^{CM\times C}\)
contains
\[
    CM\times C = C^2M
\]
trainable parameters.
\end{proof}

\noindent
\cref{prop:hypothesis_class} formalizes the regularization effect of RVP. It does not enlarge the dense hypothesis class. Instead, it restricts it to a structured subset that first aggregates prompts within each class and then applies a residual inter-class correction. This is particularly desirable in the few-shot regime, where unrestricted dense mappings are more likely to overfit.

\paragraph{Discussion.}
Taken together, the results above explain why RVP is effective. \cref{prop:text_span} shows that RVP preserves the pretrained CLIP semantic subspace. \cref{thm:projection} and \cref{cor:margin_recovery} show that its inter-class residual can explicitly remove low-rank shared semantic components that obscure fine-grained discrimination. \cref{prop:stable_margin} shows that this correction is stable because it acts in residual form. Finally, \cref{prop:hypothesis_class} shows that RVP achieves these benefits while restricting the classifier to a structured low-complexity family. Together, these properties provide a principled explanation for why RVP works well in few-shot fine-grained visual reprogramming.

\subsection{Proof of Removing Visual Embedding Normalization at Inference}
\label{app:remove_norm}

In this section, we show that the \(\ell_2\)-normalization of the visual embedding can be omitted at inference without changing the final predicted class, provided that only classification decisions are of interest.

Recall that the inference-time logits of RVP are given by
\begin{equation}
    \mathbf{z} = \hat{\mathbf{v}}^\top \hat{W},
\end{equation}
where \(\hat{\mathbf{v}} \in \mathbb{R}^D\) is the normalized visual embedding and \(\hat{W} \in \mathbb{R}^{D \times C}\) is the reparameterized classifier matrix. Let \(\mathbf{v} \in \mathbb{R}^D\) denote the corresponding unnormalized visual embedding produced by the frozen image encoder. By definition,
\begin{equation}
    \hat{\mathbf{v}} = \frac{\mathbf{v}}{\|\mathbf{v}\|_2}.
\end{equation}
Substituting this into the inference equation gives
\begin{equation}
    \mathbf{z}
    = \left( \frac{\mathbf{v}}{\|\mathbf{v}\|_2} \right)^\top \hat{W}
    = \frac{1}{\|\mathbf{v}\|_2} \mathbf{v}^\top \hat{W}.
\end{equation}
Now define the logits computed without visual normalization as
\begin{equation}
    \tilde{\mathbf{z}} = \mathbf{v}^\top \hat{W}.
\end{equation}
Then we have
\begin{equation}
    \mathbf{z} = \frac{1}{\|\mathbf{v}\|_2} \tilde{\mathbf{z}}.
\end{equation}
That is, the normalized and unnormalized logits differ only by the multiplicative factor \(1 / \|\mathbf{v}\|_2\), which is a positive scalar shared by all classes for the same sample.

Let \(z_c\) and \(\tilde{z}_c\) denote the \(c\)-th entries of \(\mathbf{z}\) and \(\tilde{\mathbf{z}}\), respectively. Then for any two classes \(i\) and \(j\),
\begin{equation}
    z_i - z_j
    = \frac{1}{\|\mathbf{v}\|_2} \left( \tilde{z}_i - \tilde{z}_j \right).
\end{equation}
Since \(\|\mathbf{v}\|_2 > 0\), multiplication by \(1 / \|\mathbf{v}\|_2\) preserves the sign of every pairwise logit difference. Therefore,
\begin{equation}
    z_i > z_j
    \quad \Longleftrightarrow \quad
    \tilde{z}_i > \tilde{z}_j.
\end{equation}
This implies that the ordering of class logits is unchanged, and hence
\begin{equation}
    \arg\max_{c \in \{1,\dots,C\}} z_c
    =
    \arg\max_{c \in \{1,\dots,C\}} \tilde{z}_c.
\end{equation}

Therefore, omitting the \(\ell_2\)-normalization of the visual embedding does not affect the final predicted class. In other words, if only top-1 classification is required, the inference rule
\begin{equation}
    \mathbf{z} = \hat{\mathbf{v}}^\top \hat{W}
\end{equation}
is equivalent to
\begin{equation}
    \tilde{\mathbf{z}} = \mathbf{v}^\top \hat{W}.
\end{equation}
The latter has exactly the same form as a standard linear classifier applied to backbone features.

We emphasize that this equivalence holds for classification decisions, but not for the absolute scale of the logits or calibrated confidence scores. Indeed, removing the normalization changes the magnitude of the logits by a sample-dependent factor \( \|\mathbf{v}\|_2^{-1} \), which may affect softmax probabilities, confidence calibration, or any downstream procedure that depends on logit scale.

\section{Implementation Details.}
\label{sec:implementation}
Following prior work~\citep{cai2025attribute,cai2025understanding}, we train the visual prompt with a learning rate of 40, momentum 0.9, using stochastic gradient descent (SGD), and a cosine annealing scheduler~\citep{loshchilov2016sgdr} for 200 epochs. For all datasets, we use a batch size of 64. For the intra-class matrix $P$ and inter-class matrix $E$, we use a learning rate of $10^{-3}$. To ensure a fair comparison, we adopt the same text descriptions as in~\citep{cai2025understanding}, using $M=20$ prompts per class. Our method does not introduce additional hyperparameters. The temperature $\tau$ is inherited from the pretrained CLIP model and kept fixed during training. 

All experiments are conducted on a single NVIDIA L40S GPU with 48 GB of memory. The full set of experiments across 11 datasets requires approximately 47.5 hours. As shown in \cref{fig:rvp_memory_fgvc_aircraft}, the training process uses about 6.37 GB of GPU memory.

\begin{figure*}[ht]
    \centering
    \begin{subfigure}[t]{\linewidth}
        \centering
        \includegraphics[width=\linewidth]{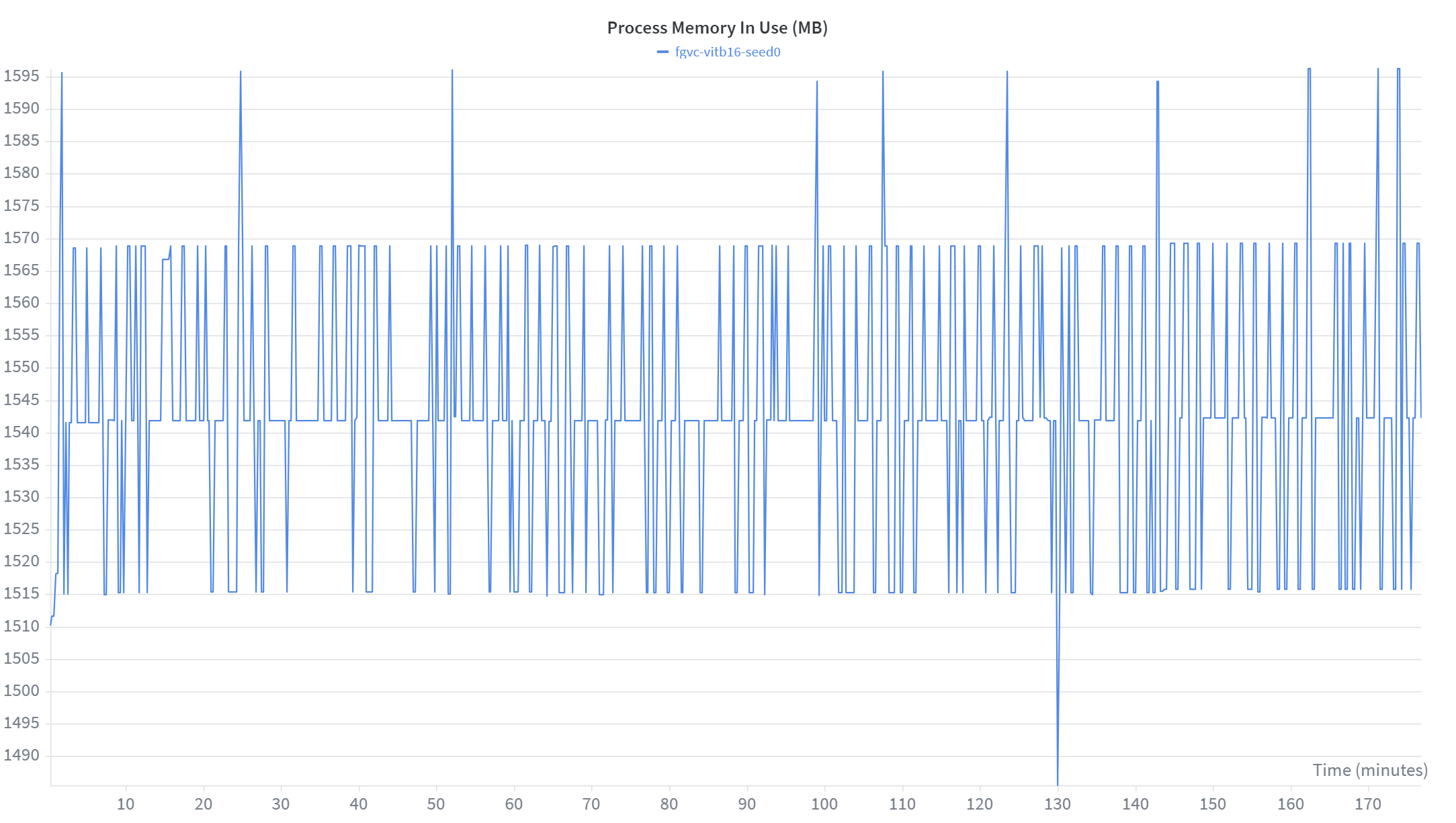}
        \caption{Host (CPU) memory usage during training}
        \label{fig:rvp_cpu_memory}
    \end{subfigure}
    \hfill
    \begin{subfigure}[t]{\linewidth}
        \centering
        \includegraphics[width=\linewidth]{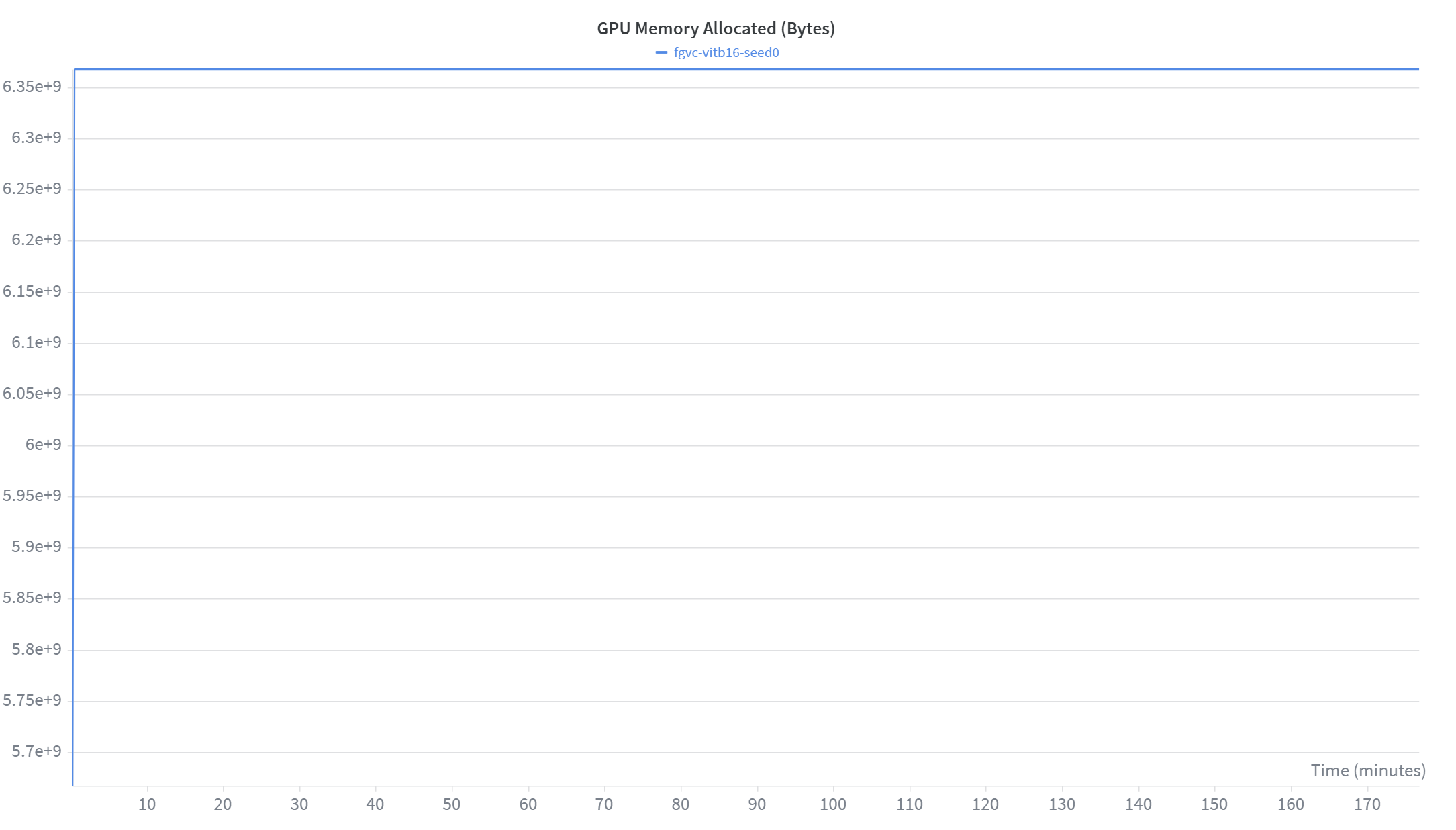}
        \caption{GPU memory usage during training}
        \label{fig:rvp_gpu_memory}
    \end{subfigure}
    \caption{Memory consumption when training RVP with a ViT-B/16-based CLIP backbone on the FGVC Aircraft dataset. Left: host memory. Right: GPU memory.}
    \label{fig:rvp_memory_fgvc_aircraft}
\end{figure*}

\section{Notations}
In this section, we summarize the abbreviations and key mathematical notations used in this paper to improve clarity.
\subsection{Abbreviations}
\begin{table}[H]
\centering
\caption{Abbreviations used in the paper}
\label{tab:rvp_abbrev}
\small
\begin{tabular}{l|l}
\toprule
\textbf{Abbreviation} & \textbf{Description} \\
\midrule
RVP & Reparameterized Inter-Class Visual Reprogramming. \\
VR & Visual Reprogramming. \\
VLM & Vision-Language Model. \\
CLIP & Contrastive Language-Image Pre-training. \\
VP & Visual Prompting / standard visual reprogramming baseline. \\
AR & Adversarial Reprogramming baseline. \\
AttrVR & Attribute-based Visual Reprogramming. \\
DVP & Decoupled Visual Prompting. \\
DVP-cls & DVP with partitions formed by unsupervised clustering of description embeddings. \\
LLM & Large Language Model. \\
PRM & Probability Reweighting Matrix used in DVP. \\
CE Loss & Cross-Entropy Loss. \\
SGD & Stochastic Gradient Descent. \\
\bottomrule
\end{tabular}
\end{table}

\subsection{Notation}
\begin{table}[H]
\centering
\caption{Generic notation in visual reprogramming}
\label{tab:rvp_notation_core}
\small
\begin{tabular}{l|l}
\toprule
\textbf{Symbol} & \textbf{Description} \\
\midrule
$f_{\rm img}$ & CLIP image encoder. \\
$f_{\rm txt}$ & CLIP text encoder. \\
$\mathcal{X}^{\rm S}$ & Source image space of the pretrained CLIP model, with $\mathcal{X}^{\rm S}\subseteq \mathbb{R}^{d_{\rm S}}$. \\
$\mathcal{X}^{\rm T}$ & Target image space for the downstream task, with $\mathcal{X}^{\rm T}\subseteq \mathbb{R}^{d_{\rm T}}$. \\
$\mathcal{Y}^{\rm T}$ & Label space of the downstream task, with $\mathcal{Y}^{\rm T}=\{1,\dots,C\}$. \\
$\mathcal{V}$ & Text space containing textual descriptions. \\
$\mathcal{Z}$ & Shared embedding space for image and text features, with $\mathcal{Z}\subseteq \mathbb{R}^{D}$. \\
$x^{\rm S}$ & Source-domain image. \\
$x^{\rm T}$ & Target-domain image. \\
$V$ & A text description in $\mathcal{V}$. \\
$y^{\rm T}$ & A downstream class label. \\
$C$ & Number of downstream classes. \\
$D$ & Embedding dimension of CLIP. \\
$d_{\rm S}$ & Input dimensionality of source-domain images. \\
$d_{\rm T}$ & Input dimensionality of target-domain images. \\
$\hat{v}$ & $\ell_2$-normalized visual embedding produced by the CLIP image encoder. \\
$\hat{t}$ & $\ell_2$-normalized text embedding produced by the CLIP text encoder. \\
$f_{\rm clip}(x^{\rm S},V)$ & CLIP similarity score between image $x^{\rm S}$ and text $V$. \\
$\tau$ & Temperature parameter in CLIP similarity computation. \\
$\mathcal{A}$ & Full set of textual descriptions used for the downstream task. \\
$\mathcal{A}(y^{\rm T})$ & Set of textual descriptions associated with class $y^{\rm T}$. \\
$M$ & Number of textual descriptions per class. \\
$a$ & A description element from $\mathcal{A}$. \\
$\mathrm{agg}(\cdot)$ & Aggregation operator over description-level similarities. \\
$[f_{\rm logits}(x^{\rm T};\mathcal{A})]_{y^{\rm T}}$ & Logit of class $y^{\rm T}$ computed from aggregated image-text similarities. \\
$|\mathcal{A}|$ & Number of textual descriptions in $\mathcal{A}$. \\
$f_{\rm in}(x^{\rm T}\mid \delta)$ & Input transformation to map a target-domain image into CLIP input space. \\
$\delta$ & Trainable visual prompt. \\
$\mathcal{D}$ & Downstream training set. \\
$\omega$ & Reweighting matrix that maps description-level similarities to class logits. \\
$\phi$ & Overall linear mapping from the normalized image embedding to downstream class logits. \\
$N$ & Number of training samples in $\mathcal{D}$. \\
\bottomrule
\end{tabular}
\end{table}

\begin{table}[H]
\centering
\caption{Notation for visual reprogramming and the linear mapping view}
\label{tab:rvp_notation_vr}
\small
\resizebox{\textwidth}{!}{%
\begin{tabular}{l|l}
\toprule
\textbf{Symbol} & \textbf{Description} \\
\midrule
$T$ & Stacked matrix of normalized text embeddings. In RVP with $C$ classes and $M$ descriptions per class, $T\in \mathbb{R}^{CM\times D}$. \\
$M_a$ & Vector of similarity scores over all textual descriptions. \\
$M_y$ & Vector of downstream class logits. \\
$P \in \mathbb{R}^{C\times M}$ & Learnable intra-class weighting matrix for aggregating attribute descriptions within each class. \\
$\tilde{P}_c$ & Softmax-normalized weight vector for the $c$-th class, obtained from the $c$-th row of $P$. \\
$\tilde{P}_{c,m}$ & Normalized weight assigned to the $m$-th textual description of class $c$. \\
$\hat{t}_{c,m}$ & Normalized embedding of the $m$-th textual description for class $c$. \\
$f_c(x^{\rm T})$ & Aggregated base logit for class $c$ before inter-class refinement. \\
$f(x^{\rm T})$ & Row vector of base logits for all classes, $f(x^{\rm T})=[f_1(x^{\rm T}),\dots,f_C(x^{\rm T})]\in\mathbb{R}^{C}$. \\
$E \in \mathbb{R}^{C\times C}$ & Learnable inter-class adjacency matrix for modeling class relationships. \\
$I$ & Identity matrix used in residual message passing. \\
$z \in \mathbb{R}^{C}$ & Final class logit vector after inter-class refinement. \\
$W_1 \in \mathbb{R}^{CM\times C}$ & Sparse routing matrix constructed from the normalized intra-class weights. \\
$\hat{W} \in \mathbb{R}^{D\times C}$ & Reparameterized linear classifier that absorbs text embeddings, intra-class aggregation, and inter-class refinement. \\
$\hat{v}^{\top}\hat{W}$ & Final inference form of RVP as a single linear projection on the normalized visual embedding. \\
$v^{\top}\hat{W}$ & Equivalent inference form for top-1 prediction when visual feature normalization is omitted. \\
\bottomrule
\end{tabular}%
}
\end{table}

\end{document}